\documentclass[journal]{IEEEtran}

\usepackage{xcolor,soul,framed} %,caption

\colorlet{shadecolor}{yellow}
\usepackage[pdftex]{graphicx}
\graphicspath{{../pdf/}{../jpeg/}}
\DeclareGraphicsExtensions{.pdf,.jpeg,.png}

\usepackage{afterpage}

\usepackage[cmex10]{amsmath}
\usepackage{amsthm}
\newtheorem{thm}{Prop}
\newtheorem*{remark}{Remark}

\let\labelindent\relax
\usepackage{enumitem}

\usepackage{array}
\usepackage{mdwmath}
\usepackage{mdwtab}
\usepackage{eqparbox}
\usepackage{url}
\usepackage{booktabs}  % 用于创建美观的水平线
\usepackage{multirow}  % 用于合并单元格

\usepackage{amsmath,amssymb,amsfonts}
\usepackage{algorithm, algorithmic}
\usepackage{graphicx}
\usepackage{textcomp}
\usepackage{xcolor}
\usepackage{subfigure}
\usepackage{multirow}
\usepackage{makecell}
\usepackage{bm}
\usepackage{tabularx}
\usepackage{booktabs}

\newcommand{\subl}[1]{{#1}_{\mathrm{left}}}

\newcommand{\subr}[1]{{#1}_{\mathrm{right}}}

\newcommand{\Ll}[0]{\subl{L}}
\newcommand{\Lc}[0]{L}
\newcommand{\Lr}[0]{\subr{L}}
\newcommand{\Fl}[0]{\subl{F}}
\newcommand{\Fc}[0]{F}
\newcommand{\Fr}[0]{\subr{F}}
\newcommand{\gl}[0]{\subl{g}}
\newcommand{\gc}[0]{g}
\newcommand{\gr}[0]{\subr{g}}
\newcommand{\bgl}[0]{\subl{bg}}
\newcommand{\bgc}[0]{bg}
\newcommand{\bgr}[0]{\subr{bg}}

\newcommand{\bLl}[0]{d_{\Ll}}
\newcommand{\bLc}[0]{d_{\Lc}}

\newcommand{\bFl}[0]{d_{\Fl}}

\newcommand{\bE}[0]{d_E}

\newcommand{\vsafel}[0]{v^\mathrm{s}_\mathrm{left}}
\newcommand{\vsafec}[0]{v^\mathrm{s}}
\newcommand{\vsafer}[0]{v^\mathrm{s}_\mathrm{right}}

\newcommand{\asafec}[0]{a^\mathrm{s}}

\newcommand{\vstarl}[0]{v^*_\mathrm{left}}
\newcommand{\vstarc}[0]{v^*}
\newcommand{\vstarr}[0]{v^*_\mathrm{right}}

\newcommand{\spdlim}[0]{\operatorname{spdlim}}

\newcommand{\Reff}[0]{R^{}_{\mathrm{eff}}}
\newcommand{\Rcomf}[0]{R^{}_{\mathrm{comf}}}
\newcommand{\Rlc}[0]{R^{}_{\mathrm{discr}}}
\newcommand{\Rmlc}[0]{R^{}_{\mathrm{route}}}

\newcommand{\weight}[0]{\xi}

\newcommand{\wcomf}[0]{\weight^{}_{\mathrm{comf}}}

\newcommand{\wlc}[0]{\weight^{}_{\mathrm{discr}}}
\newcommand{\wmlc}[0]{\weight^{}_{\mathrm{route}}}

\newcommand{\dt}[0]{r}
\newcommand{\rem}{\operatorname{Rem}}

\newcommand{\clip}[3]{\operatorname{clip}(#1,#2,#3)}
\newcommand{\argmax}[0]{\operatorname{argmax}}

\newcommand{\method}[0]{SECRM-2D}

\usepackage[hidelinks]{hyperref}
\usepackage{cite}

\begin{document}
\title{SECRM-2D: RL-Based Efficient and Comfortable Route-Following Autonomous Driving with Analytic Safety Guarantees}

  \author{Tianyu Shi,
      Ilia Smirnov,
      Omar ElSamadisy,
      Baher Abdulhai
\vspace{-0.5cm}
\thanks{The authors are with the Dept. of Civil and Mineral Engineering. University of Toronto, Canada. ty.shi@mail.utoronto.ca, ilia.smirnov@utoronto.ca, omar.elsamadisy@mail.utoronto.ca, baher.abdulhai@utoronto.ca}%
\thanks{Omar ElSamadisy is on leave with Dept. of Electronics \& Communications Engineering, Arab Academy for Science and Technology, Egypt. omar.elsamadisy@aast.edu}
}

%   \thanks{Manuscript received July 10, 2012. \hl{This paper is an expanded paper from the IEEE MTT-S Int. Microwave Symposium held on June 17-22, 2012 in Montreal, Canada.} This work was funded in part by the Office of Naval Research under the Defense Advanced Research Projects Agency (DARPA) Microscale Power Conversion (MPC) Program under Grant N00014-11-1-0931, and in part by the Advanced Research Projects Agency-Energy (ARPA-E), U.S. Department of Energy, under Award Number DE-AR0000216.}

%\markboth{IEEE TRANSACTIONS ON MICROWAVE THEORY AND TECHNIQUES, VOL.~60, NO.~12, DECEMBER~2012
%}{Shi \MakeLowercase{\textit{et al.}}: Efficient and Comfortable RL-Based Autonomous Driving with Route-Following and Analytic Safety Guarantees}

% ====================================================================
\maketitle

% === ABSTRACT ====================================================================
% =================================================================================
\begin{abstract}
Over the last decade, there has been increasing interest in autonomous driving systems. Reinforcement Learning (RL) shows great promise for training autonomous driving controllers, being able to directly optimize a combination of criteria such as efficiency comfort, and stability. However, RL-based controllers typically offer no safety guarantees, making their  readiness for real deployment questionable. In this paper, we propose \method~(the Safe, Efficient and Comfortable RL-based driving Model with Lane-Changing), an RL autonomous driving controller (both longitudinal and lateral) that balances optimization of efficiency and comfort and follows a fixed route, while being subject to hard analytic safety constraints. The aforementioned safety constraints are derived from the criterion that the follower vehicle must have sufficient headway to be able to avoid a crash if the leader vehicle brakes suddenly. We evaluate \method~against several learning and non-learning baselines in simulated test scenarios, including  freeway driving, exiting, merging, and emergency braking. Our results confirm that representative previously-published RL AV controllers may crash in both training and testing, even if they are optimizing a safety objective. By contrast, our controller \method~is successful in avoiding crashes during both training and testing, improves over the baselines in measures of efficiency and comfort, and is more faithful in following the prescribed route. In addition, we achieve a good theoretical understanding of the longitudinal steady-state of a collection of \method~vehicles.
\end{abstract}

% === KEYWORDS ====================================================================
% =================================================================================

% \fi
%\IEEEpeerreviewmaketitle

%\vspace{-0.3cm}

% === I. INTRODUCTION =============================================================
% =================================================================================
\section{Introduction}
%The goal is to optimize a combination of factors: travel time, safety, and comfort. Travel time aims to make the vehicle move as fast and safely as possible, ensuring efficiency. Safety considerations are embedded to ensure secure driving, especially during critical scenarios like exits, merges, and intersections. Comfort is optimized by minimizing jerk, enhancing the overall ride experience. This unified objective function encapsulates various aspects of driving behavior, enabling a comprehensive optimization strategy. For our proposed \method~model, we used a unified RL framework to learn those objectives. 

%In addition to car-following, lane change control is a major task in autonomous driving that requires careful decision-making to ensure safety, efficiency, and comfort. Efficiency entails driving faster and closer to the leader to minimize travel time and maximize flow. Safety entails not crashing, while comfort entails avoiding jerky movement with rapidly changing accelerations.

Achieving autonomous driving is an important goal in the development of intelligent transportation systems. Reinforcement Learning (RL) is a general optimization technique that directly optimizes a particular objective, without requiring any special problem structure (such as linearity or convexity). RL is becoming mature, and RL-based systems have been able to surpass human performance in several domains, such as game-playing \cite{dqn} and data-center cooling \cite{lazic2018data} in recent years. Because autonomous driving involves optimization of a diverse combination of objectives, including efficiency, comfort, route-following, and safety, RL is an attractive approach to training autonomous vehicle (AV)  controllers (more explicit definitions of the optimization objectives chosen for this work will be given in Sections~\ref{sec:formulation} and~\ref{sec:method}). 

In addition, model-free RL algorithms can learn from data without requiring an explicit model of the environment. Formulating a model often requires introducing crucial simplifications and heuristics. By contrast, model-free RL-based control algorithms have the potential to respond to  hard-to-model problem dynamics.

However, RL controllers and especially Deep RL controllers (that involve neural networks as function approximators in order to approach large-scale and complex problems) can be difficult to interpret. The black-box nature of the controllers makes certifiable behaviour difficult to achieve with the standard methods. One of the main current issues in applying RL to car following and lane change control is the need to ensure safety~\cite{li2022decision}. 
%Safety is of utmost importance in autonomous driving, and RL algorithms must be designed to prioritize and adhere to safety guidelines. RL agents need to learn policies that not only optimize lane change maneuvers but also take into account the safety of the vehicle and other road users. 
Developing robust safety mechanisms is crucial to prevent hazardous situations. 

%Additionally, RL algorithms need to be trained and evaluated in a variety of realistic and challenging traffic scenarios to ensure their generalization and ability to handle diverse situations. Balancing these safety, efficiency, and comfort objectives is essential to instill trust and confidence in RL-based lane change control systems and pave the way for their widespread adoption in autonomous driving.

In previous work SECRM~\cite{secrm}, we developed a longitudinal car-following RL controller with safety guarantees. In this paper, SECRM is extended with lateral (lane-changing) control actions. Lane-change maneuvers may be categorized as being either discretionary or mandatory. A lane change maneuver is defined as discretionary when it is not explicitly required or mandated by external factors such as road closures or merging lanes. In discretionary lane changes, the decision to change lanes is primarily based on goals such as optimizing travel time, improving comfort, or adapting to traffic conditions~\cite{zhang2019discretionary}. By contrast, lane changes that are required by external factors are defined as mandatory~\cite{xi2020efficient}. We are the first to propose an RL controller that can perform both discretionary and mandatory lane changes. The main contributions are:
\begin{itemize}[wide]
    \item We propose 
    \method, an RL-based controller for autonomous driving, which optimizes efficiency and comfort, and follows a prescribed route, while incorporating hard analytic safety constraints.

\item Longitudinal safety constraints from \cite{secrm} are extended to formulate lane-change safety constraints.

\item Discretionary and mandatory lane changes are combined into a single unified optimization objective. To the best of our knowledge, this is the first RL AV controller to handle both discretionary and mandatory lane changing on equal footing.

%\item We evaluate the proposed method in a various experiments demonstrate that several previously-published RL AV controllers are prone to crashes, while \method~successfully avoids crashes during training and testing. \method~outperforms both a standard rules-based model (IDM with MOBIL) and the previous RL controllers in terms of efficiency, comfort and route-following measures.

\item In addition, the steady-state behavior of a group of \method~ vehicles is studied. We demonstrate that these vehicles naturally form a steady-speed platoon as a consequence of efficiency optimization. An explicit formula is derived for the gap between vehicles, as a function of platoon speed and parameters like maximal braking rate and reaction time. The findings are validated in simulation.
\end{itemize}
This work advances the promise of RL to train performant AV controllers, while incorporating strong, explainable and  explicit safety constraints.

%\section{Related work}
%Lane changing is one of the most common activities in freeway driving. The main purpose of lane change is to bypass slower vehicles or switch to the target routes. Based on ~\cite{FHAGuide1995}, there are mainly two types of lane change, i.e., mandatory lane changes (MLCs) and discretionary lane changes (DLCs). In the field of traffic flow management, understanding and optimizing mandatory and discretionary lane changes are essential topics.

%Many Simulators  have been proposed to simulate lane change behavior. From the perspective of microscopic traffic simulators. AIMSUN describes a vehicle's intention to change lanes based on its imperative, preference, and feasibility of lane changes~\cite{AIMSUNManual2002} . VISSIM also categorizes lane changes as either free lane changes or necessary lane changes, corresponding to DLCs and MLCs, respectively~\cite{VISSIMManual2007}.  SUMO handles lane changes based on various motivations, including strategic changes needed to reach the next edge on the route, discretionary changes made for comfort or efficiency, and mandatory changes enforced by rules or traffic conditions. Lane changes are influenced by factors such as safety, collision avoidance, and distance to target lanes. ~\cite{erdmann2015sumo}.

Many researchers investigated lane change topics. For example, Pan et al.~introduced a mesoscopic multilane freeway model integrating both MLC and DLC maneuvers which was further calibrated using real-world data \cite{pan2016modeling}. Furthermore, Shi et al.~proposed a hierarchical structure based on reinforcement learning to make the autonomous vehicle learn to decide when to lane change and how to lane change \cite{shi2019driving}.  Hoel et al. combined planning and learning with Monte Carlo Tree Search for the highway exit case \cite{hoel2019combining}. However, they haven't given enough consideration to the safety issue.

Furthermore, Cao et al.~investigated a learning-enhanced highway-exit planner to adapt to urgent exiting cases. They proposed safety and traffic rule constraints to construct the safe margin of speed \cite{cao2020highway}. Ye et al.~integrates penalty for collision and near-collision into reward function to enhance safety \cite{ye2020automated}. Udatha et al.~embed probabilistic control barrier functions into reinforcement learning to improve safety in ramp merging \cite{udatha2023reinforcement}. However, they cannot adapt their safety formulation to both mandatory and discretionary lane change, and they haven't provided much theoretical understanding of the safety criteria.

\section{Problem formulation}\label{sec:formulation}
In this section, the problem to be solved is formulated. After an informal discussion, the notation is set and the requirements are formalized.

Our goal is to train a controller for a single vehicle that is driving along a road consisting of one or multiple lanes. The controller acts in discrete regular time-steps (typically every 0.1s). The duration of a single time-step is the controller's reaction time $r$. We measure time in units of $r$. Every time-step, the controller sets both the vehicle's continuous \emph{longitudinal (within-lane) acceleration} and the discrete \emph{lateral (lane-change) movement} (either keep on the current lane, shift one lane to the left, or shift one lane to the right).

It is assumed that the controlled vehicle obtains a fixed route from an independent routing module. There are many possible choices for the routing module:  for example, this could be a centralized module that routes multiple vehicles to optimize network-level metrics (for example, a module based on Google Maps), or a decentralized module that simply selects the shortest-distance route from origin to destination.

The controlled vehicle performs both \emph{discretionary} lane-changes undertaken to improve its state (for example, lane-changing to pass a slow leader vehicle), and \emph{mandatory} lane-changes that are necessary for the vehicle to follow the fixed route (e.g. at bifurcations).

The controller operates under hard safety constraints that are derived from the Vienna convention on road traffic, similar to the analysis carried out in \cite{vienna}. The Vienna convention  defines a vehicle to be safe if and only if the gap between the vehicle and its leader vehicle is positive, and in the event that its leader starts braking with its maximal deceleration, the gap between the vehicle and its leader is sufficient for the vehicle to be able to react and stop without crashing. More restrictively, the gap between the leader and the ego vehicle will be required to be at least $\epsilon$ meters after coming to a stop, where $\epsilon$ is a parameter (typically, 2 m).

In the literature, it is common to define safety criteria based on a Time-to-Collision (TTC) threshold \cite{zhu2020safe, yen2020proactive, shi2022bilateral}. We have argued in favor of the Vienna-convention-based criteria over TTC-threshold-based criteria in previous work \cite{secrm}. A TTC threshold criterion is not sufficient to provide safety guarantees in the worst case and is based on an ad-hoc threshold choice.

The overall objective is to maximize a weighted combination of measures of efficiency and comfort, subject to obeying the aforementioned hard safety constraints and following the prescribed route communicated to the controlled vehicle. Discomfort is measured by the vehicle's jerk, while efficiency is measured by the vehicle's speed.

We split the task of maximizing efficiency (i.e. maximizing the vehicle's speed) into two sub-tasks: 1) driving as closely as possible to the maximal safe speed within the vehicle's current lane, 2) if it is possible to increase the maximal safe speed by changing lanes, carrying out the lane change, subject to lane-change safety constraints formulated below.

To connect our notion of efficiency (maximizing speed) with that commonly found in the literature (minimizing headway), we note that if the ego vehicle is following a leader closely enough that the ego vehicle's speed is constrained by the leader (and not the speed limit), driving at the maximal longitudinal safe speed is equivalent to driving at the lowest safe headway, so that maximizing the speed is equivalent to minimizing the front-gap subject to safety constraints.

We now proceed to set the notation used throughout the paper and make the above requirements more formal.

\subsection{Notation}

\afterpage{
    \begin{figure}[t!]
    \centering
    {\includegraphics[width=0.9\linewidth]{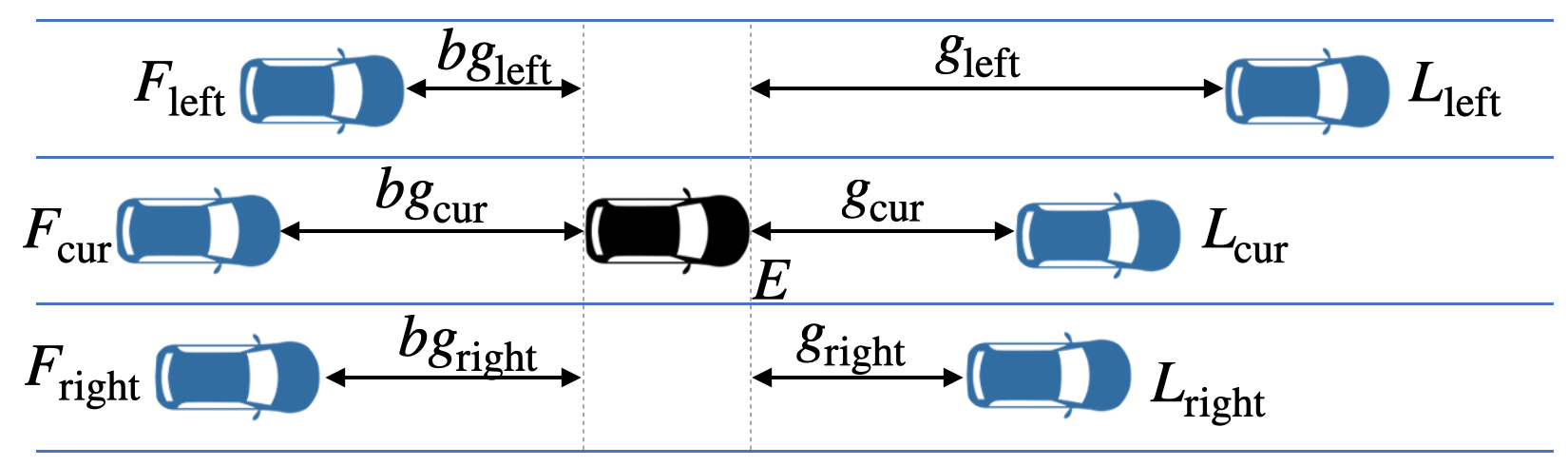}}
    \caption{Notation} 
    \vspace{-0.5cm}
    \label{fig:vars}
    \end{figure}
}

The  notation is illustrated in Figure~\ref{fig:vars}.
The ego vehicle is denoted as $E$, with the leader vehicle $\Lc$ in front and the follower vehicle $\Fc$ behind on the same lane. For a potential left lane change, we define the left-leader and left-follower as $\Ll$ and $\Fl$, respectively. Analogously, the right-leader and right-follower are denoted by $\Lr$ and $\Fr$ for a potential right lane change.

The distances are denoted as follows: front-gap between $E$ and $\Lc$ as $\gc$, back-gap between $E$ and $\Fc$ as $\bgc$. For the left lane, the front-gap to $\Ll$ and back-gap to $\Fl$ are $\gl$ and $\bgl$, respectively; for the right lane, these gaps to $\Lr$ and $\Fr$ are $\gr$ and $\bgr$.

In the event that there is no leader or follower, the appropriate gap is set to $\infty$. In the event that $E$ is driving along the left edge of the road, we set $\gl = \bgl = -1$, and analogously if $E$ is driving along the right edge, we set $\gr = \bgr = -1$.

The maximal acceleration is denoted by $a_\bullet$, where $\bullet$ denotes one of the vehicles defined above. For example, $a_E$ is the maximal acceleration of the ego vehicle. Similarly, the maximal deceleration is denoted by $d_\bullet$; by convention, $d_\bullet$ is a positive number. The number $r_\bullet$ denotes the reaction time of the corresponding vehicle. For simplicity, we write $r_E = r$. The current speed limit is denoted by $\spdlim_\bullet$.

\subsection{Safety constraint formulation} The longitudinal safety constraint is derived as in \cite{gipps} and \cite{secrm}.
Let the ego vehicle's reaction time $r$, minimal stopped gap $\epsilon$ and speed $v_E(t+1)$ in the next time-step be given. Using the laws of kinematics, it may be derived that: the ego vehicle has a sufficiently large front gap to be able to stop safely in the event of an emergency stop by the leader if and only if
\begin{equation} \label{eq:safe-gap}
    \gc(t) \geq \frac{v_E(t) + v_E(t+1)}{2}\,r + \frac{v_E^2(t+1)}{2\bE} - \frac{v_{\Lc}^2(t)}{2\bLc} + \epsilon
\end{equation}
To avoid negative values on the right-hand side of  Equation~\ref{eq:safe-gap}, in this work we adopt the \emph{defensive principle} of assuming that $\bE \leq d_{L_\bullet}$ for any leader vehicle $L_\bullet$.

Using Equation~\ref{eq:safe-gap}, treating $\gc(t)$ as known and $v_E(t+1)$ as unknown, it follows from properties of quadratic polynomials that the ego vehicle's maximal safe speed at the next time-step is given by
\begin{equation} \label{eq:safe-speed}
\resizebox{.9\hsize}{!}{%
$
    -\frac{r\bE}{2} + \sqrt{\left(\frac{r\bE}{2}\right)^2 -2\bE \left( \frac{rv_E(t)}{2} - \frac{v_{\Lc}^2(t)}{2\bLc} - \gc(t) + \epsilon \right)}
$}
\end{equation}
The maximal safe speed expressed in Equation~\ref{eq:safe-speed} is denoted by $\vsafec(t+1)$.
If the ego vehicle satisfies $v_E(t) \leq \vsafec(t)$ for all $t$, then (\ref{eq:safe-gap}) is satisfied for all $t$, and therefore the front-gap is safe for all $t$.

Using similar derivations for imaginary copies of $E$ shifted one lane to the left or right, we obtain analogous expressions for maximal safe speeds $\vsafel(t+1)$ and $\vsafer(t+1)$ on the adjacent lanes. The expressions have the same form as (\ref{eq:safe-speed}), with $v_{\Lc}$ and $\gc$ replaced by the appropriate analogues.

We now formulate the safety constraints that must be satisfied by \method:

\subsubsection*{Acceleration safety constraint} For all time-steps $t$, the ego vehicle's speed must satisfy $v_E(t) \leq \vsafec(t) $. Equivalently, the ego vehicle's acceleration must satisfy
\begin{equation} \label{eq:acc-constraint}
    a_E(t) \leq \asafec(t) := \frac{\vsafec(t+1) - v_E(t)}{r} \quad \mbox{for all $t$}
\end{equation}
\subsubsection*{Lane-change safety constraints} The ego vehicle is not permitted to perform a left lane-change if doing so would violate the safe-gap inequality (\ref{eq:safe-gap}) for either the $(\Fl, E)$ or the $(E, \Ll)$ vehicle pair.

To elaborate further, we define the constraints for a left lane-change. This involves the assumption that both the ego vehicle and the left-follower maintain their current speeds in the subsequent time step. Consequently, this assumption simplifies equation (\ref{eq:safe-gap}) to a set of two conditions
\begin{equation} \label{eq:lc-constraint}
    \begin{split}
         \gl(t) &\geq v_E(t)\,r + \frac{v_E^2(t)}{2\bE} - \frac{v_{\Ll}^2(t)}{2\bLl} + \epsilon \\
         \bgl(t) &\geq v_{\Fl}(t)\,r^{}_{\Fl} + \frac{v_{\Fl}^2(t)}{2\bFl} - \frac{v_E^2(t)}{2\bE} + \epsilon
    \end{split}
\end{equation}

If either of the inequalities in (\ref{eq:lc-constraint}) does not hold at time-step $t$, then the left lane-change action is prohibited at time-step $t$. The right lane-change action is subject to analogous constraints, replacing $\Fl$ by $\Fr$ and $\Ll$ by $\Lr$ in the inequalities.

\subsection{Efficiency and target speeds}
As motivated above, the efficiency of a vehicle is measured by its average speed. A vehicle's speed is bounded above by two separate factors. First, the speed is bounded above by the road's speed limit. Second, in the presence of a close leader the speed of the ego vehicle is bounded above by $\vsafec(t)$. Motivated by this, the target speeds are defined as
\begin{equation} \label{eq:target-speeds}
\centering
\begin{split}
    \vstarl(t) &= \min(\vsafel(t), \spdlim_E(t)) \\
    \vstarc(t) &= \min(\vsafec(t), \spdlim_E(t)) \\
    \vstarr(t) &= \min(\vsafer(t), \spdlim_E(t))
\end{split}
\end{equation}
Timestep-by-timestep inefficiency is measured by the difference of the current speed with the highest of the three target speeds in Equation~\ref{eq:target-speeds}. We note that the efficiency objective handles both car-following mode (when there is a close enough leader that constrains the ego vehicle's speed) and speed-control mode (when the ego vehicle's speed is constrained by the speed limit).

\subsection{Comfort definition}
To increase the comfort of the passengers, the controller seeks to minimize the cumulative squared jerk over the controlled vehicle's trajectory. The cumulative squared jerk over a trajectory with horizon $T$ is given by
\begin{equation}
\sum_{t=1}^{T} \left(\frac{a_E(t+1) - a_E(t)}{r}\right)^2
\end{equation}

\subsection{Route-following}
\label{route}
It is assumed that the vehicle is provided a fixed route, consisting of road sections $s_1, \dots, s_n$ indexed along the direction of motion. Each section $s_i$ consists of several lanes $\ell_{i1}, \dots, \ell_{im}$ indexed from the right edge of the section toward the left along the direction of motion. Each lane has a boolean label: whether or not the lane is part of the route. For example, for a freeway off-ramp, if the vehicle's route is to continue on the freeway, the lanes that lead to the  off-ramp are not part of its route. Figure~\ref{fig:route-spec} displays an example route in a roundabout (with lanes that are part of the route labeled with arrows).

A controlled vehicle that is far from the end of its current section may need to use lanes that are not part of its route in order to improve its efficiency, but the risk of missing the correct turn becomes higher the closer the vehicle is to the end of the section. The controlled vehicle must balance between improving efficiency through discretionary lane-changes and minimizing the risk of missing its route. The penalty is defined precisely in Section~\ref{sec:method}.

\afterpage{
\begin{figure}[t]
\centering
{\includegraphics[width=0.9\linewidth]{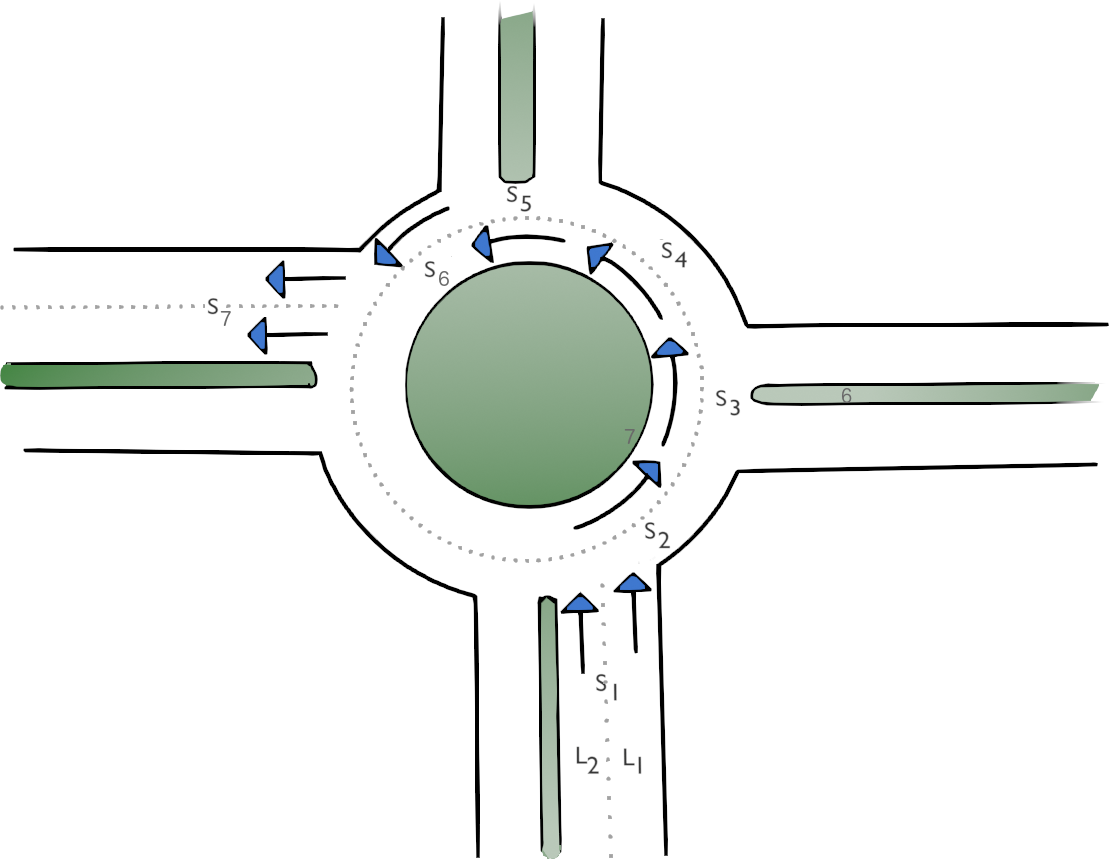}}
\caption{Example of a route passing through a roundabout. The lanes that are part of the route are indicated by an arrow. Thus, both lanes on the initial section $s_1$ and the final section $s_7$ are part of the route. For sections $s_2, s_3, s_4$ and $s_5$, only the interior lane is part of the route (as the vehicle does not take the first two exits). For section $s_6$, only the exterior lane is part of the route. The controlled vehicle may choose trajectories that use lanes that are not part of its route, but doing so will accrue a penalty (the form of the penalty is given in Section~\ref{sec:method}). Therefore, the learned controller must learn to balance between reward coming from discretionary lane-changes (for example, bypassing slow or broken down vehicles) and penalty accumulated from violating the route-following objective.}
\label{fig:route-spec}
\end{figure}
}

\section{\method}
\label{sec:method}
\subsection{Formalization of the control problem as an optimization problem in an SCMDP} 
Following the terminology of \cite{safe-rl-survey}, the problem formulated in Section~\ref{sec:formulation} may be formalized as an optimization problem within a Statewise-Constrained Markov Decision Process (SCMDP). By definition, an SCMDP is a tuple $(S, A, R, P, \gamma, \rho, C, w)$. Here, $S$ is a set called the state space, $A$ is a set called an action space, $R(s,a,s')$ is the reward model, $P(s'|s,a)$ is the transition model, $\gamma \in [0,1)$ is the discount factor, $\rho(s)$ is the initial state distribution, $C = (C_1, \dots, C_n)$ is a tuple of cost functions, where each $C_i$ is a function of $(s,a,s')$, and $w = (w_1, \dots, w_n)$ is a tuple of real-valued weights.

At each environment step, the agent starts from state $s \in S$, takes an action $a \in A$, transitions to state $s'$ with probability $P(s'|s,a)$, and receives reward $R(s,a,s')$. It is implicit in the notation that the transition model and reward model depend only on $s,a$ (and not on past history of states and actions), which is called the Markov assumption. The function $\rho$ describes the probability distribution of initial states.

A policy is a mapping $\pi(a|s)$ that sends a state to a probability distribution over actions. The tuples of cost functions and weights are used to restrict the set of feasible policies:
    \[ \Pi_{C,w} = \left\{ \pi \left| \begin{aligned}&\forall i \in \{1,\dots,n\}, \forall s \in S \quad C_i(s,a,s') \leq w_i \\ &\mbox{where } a \sim \pi(\cdot|s),\ s' \sim P(\cdot|s,a)\end{aligned}\right. \right\} \]
The agent's objective is to find a feasible policy ($\pi \in \Pi_{C,w}$) that maximizes the cumulative discounted reward 
\begin{equation}
\label{eq:J}
    J(\pi) = E_{\tau \sim \rho, \pi, P}\left( \sum_{t=0}^\infty \gamma^t R(s_t,a_t,s_{t+1}) \right)
\end{equation}
where $\tau = s_0,a_0,s_1,a_1,\dots$ denotes a state-action trajectory, and $\tau \sim \rho, \pi, P$ denotes that $s_0 \sim \rho$, $a_t \sim \pi(\cdot|s_t)$  and $s_{t+1} \sim P(\cdot|s_t,a_t)$. Therefore, $\gamma$ controls how myopic an agent is; $\gamma$ close to 0 makes the agent more short-sighted, while $\gamma$ close to 1 makes the agent plan over a longer effective horizon.

\subsection*{State space}
The ego vehicle's state representation is illustrated in Figure~\ref{fig:state-space2}. Only vehicles within a fixed scan radius $D$ are considered. For each lane within the scan range, we record features of $N_{front}$ closest vehicles per lane in front of the ego vehicle and $N_{back}$ closest vehicles per lane behind the ego vehicle.

For the ego vehicle, the state representation contains its longitudinal distance from the start of the current section, speed, previous value of acceleration, index of current section within route, lane index, lateral speed, and a boolean map over whether or not a lane with a given index is part of the ego vehicle's route on the current section. For the other vehicles, the representation is simpler: it contains the longitudinal distance from the ego vehicle, speed, and acceleration (the vehicles are sorted according to the lane they occupy and distance from the ego vehicle along their lane).
\afterpage{
\begin{figure}[t]
\centering
{\includegraphics[width=0.9\linewidth]{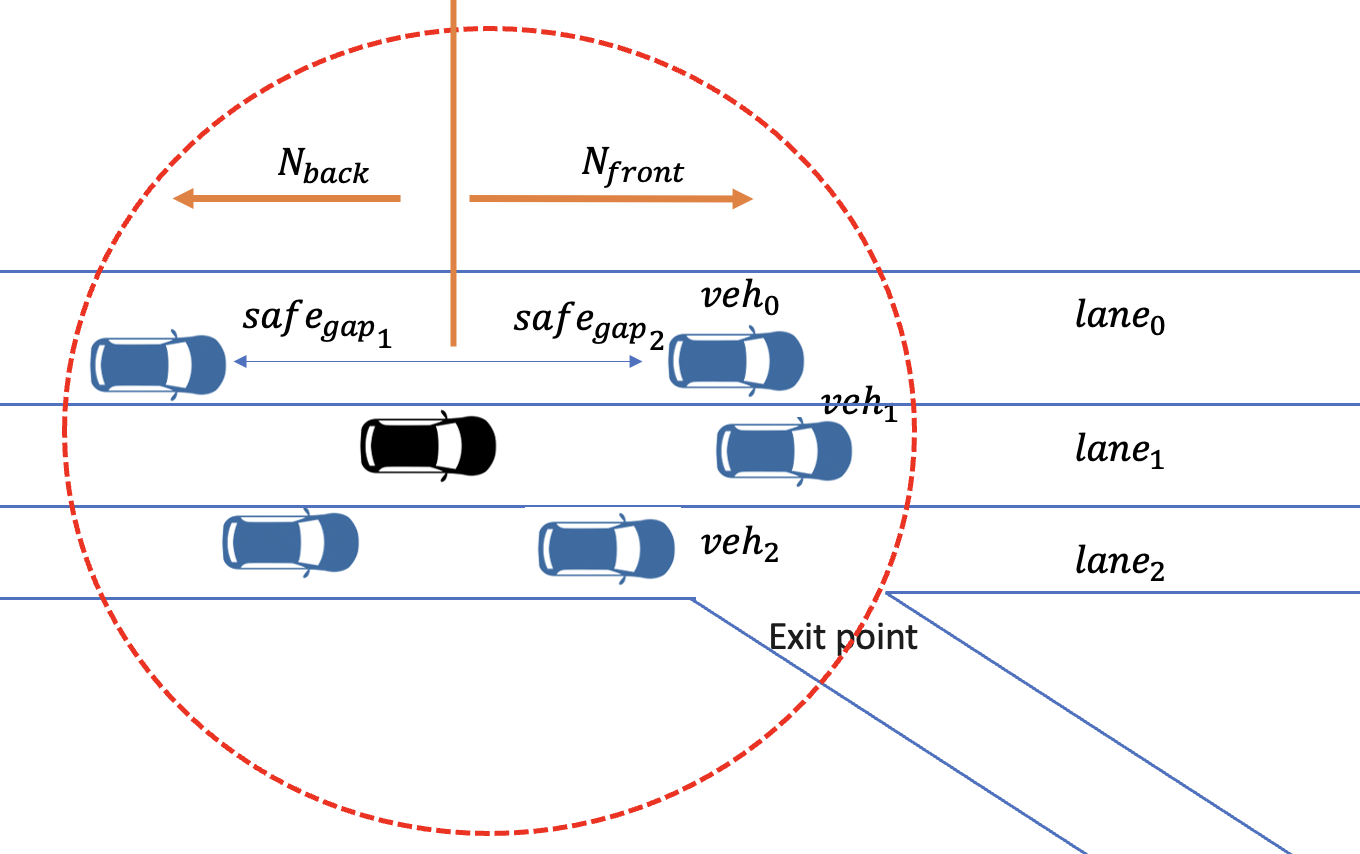}}
\caption{The state representation of the ego vehicle. The black vehicle is the ego vehicle; the blue vehicles are the surrounding vehicles.} 
\label{fig:state-space2}
\end{figure}
}

\subsection*{Action space} The action space factors as a 2-tuple $(\alpha,\ell)$, where $\alpha$ is an acceleration action and $\ell$ is a lane-change action.

The acceleration action is continuous and lies in the closed interval $[-\bE, a_E]$. The acceleration action is subject to additional statewise constraint (\ref{eq:acc-constraint}),
which is enforced through a cost function. The lane-change action is discrete, and in general lies in the set $\{\mbox{left, cur, right}\}$. The lane-changes are subject to constraints (\ref{eq:lc-constraint}), also enforced through cost functions (please see the example at the end of this section).

\subsection*{Reward} The reward is a weighted sum of four terms: 
\subsubsection*{Efficiency (longitudinal)} The inefficiency penalty is equal to the difference between the current speed and the target speed:
\begin{equation*} \label{eq:reward-eff}
    \Reff(s, a, s') = - \frac{\operatorname{abs}(\vstarc(s) - v_E(s))}{\vstarc(s)}
\end{equation*}
where $\operatorname{abs}(\cdot)$ denotes the absolute value function. (The notation $\vstarc(s)$ for the target speed on the current lane was introduced in Equation~\ref{eq:target-speeds}.)

When the ego vehicle's speed is constrained by the speed limit (the ego vehicle is in speed-control mode), speeding is allowed, but penalized.
%When $\vstarc(s) = \vsafec(s)$, we know that $v_E(s) \leq \vstarc(s)$ by the safety constraint; on the other hand, if $\vstarc(s) = \spdlim_E(s)$, $v_E(s)$ may exceed $\vstarc(s)$ (in which case the ego vehicle is speeding), but receives a penalty proportional to the difference between $v_E(s)$ and $\vstarc(s)$.

%(In most cases, the denominator normalizes $\Reff$ to lie within the interval $[-1,0]$. The possible exceptions are cases of extreme speeding, in which cases the agent receives a higher penalty than $-1$. Such cases have not arisen in our experiments.)

\subsubsection*{Comfort} The passengers' discomfort is measured as the square of the jerk of the ego vehicle in the last time-step. We remind the reader that the jerk of a state-transition $s,a,s'$ is defined as
$j_E^{}(s') = (a_E(s')-a_E(s))/r$. The comfort penalty is then defined by
\begin{equation*} \label{eq:rew-comfort}
    \begin{split}
    \Rcomf(s,a,s') &= -\left( \frac{j_E(s')}{(a_E + \bE)/r} \right)^2 \\
    &= -\left(\frac{a_E(s') - a_E(s)}{a_E + \bE}\right)^2
    \end{split}
\end{equation*}
%(Because $j_E(s') \leq (a_E + \bE)/r$, the denominator normalizes $\Rcomf$ to lie within the interval $[-1,0]$.)
\subsubsection*{Discretionary lane-changing} Let $\ell$ denote the lane-change part of the action. The discretionary lane-change reward is given by
\begin{equation} \label{eq:dlc}
\resizebox{.9\hsize}{!}{%
    $\displaystyle \Rlc(s,a,s') = \frac{1}{\vstarc} 
    \begin{cases}
        C(\vstarl, \vstarc) (\vstarl - \vstarc), & \ell = \mathrm{left} \\
        C(\vstarr, \vstarc) (\vstarr - \vstarc), & \ell = \mathrm{right} \\
        0, & \ell = \mathrm{cur}
    \end{cases}$
    }
\end{equation}
where for any $v_0, v_1$ we define
\begin{align*}
    C(v_0,v_1) &= \frac{1 - \gamma^{T(v_0,v_1)}}{1-\gamma}, \quad \mbox{and where} \\
    T(v_0,v_1) &=  \operatorname{round}\left(\frac{\operatorname{abs}\left(v_0-v_1\right)}{a_E} \cdot \frac{1}{r} \right)
\end{align*}
and where the state-dependence of the target speeds is omitted for readability ($\gamma$ is the discount factor). 

We now motivate the chosen form of $\Rlc$.

We would like to reward the agent for changing to a lane with a higher target speed, and penalize the agent for changing to a lane with a lower target speed (as part of the efficiency optimization). This target-speed-maximization formulation compactly encodes the appropriate action for several different scenarios, including scenarios when there is a faster-moving leader on one of the adjacent lanes, and scenarios when there is no leader on an adjacent lane.

  For brevity, we shall discuss left lane-changes only; right lane-changes are analogous. The previous paragraph motivates the initial choice of $R = \vstarl - \vstarc$ (which may be negative) as the discretionary lane-changing reward. However, we observe that if the agent switches to a lane with $\vstarl > \vstarc$, in the consequent time-steps the agent will receive larger efficiency penalties (than it would have if the lane-change was not performed) until the speed difference $v_E - \vstarl$ is less than or equal to $v_E - \vstarc$. In other words, the cumulative reward the agent will receive may be negative even if the agent switches to a lane with a higher target speed.
  
  Therefore, the immediate lane-change reward is boosted to compensate for the penalty that gets accrued after performing the lane-change. The boost coefficient $C(v_0,v_1)$ may be chosen empirically. Instead, we approximate as follows. In the ideal case, it will take approximately $T = T(\vstarl, \vstarc)$ time-steps for the ego vehicle to make up the speed difference. Taking discounting into account, we choose the immediate reward to be $R + \gamma R + \gamma^2 R + \cdots + \gamma^{T-1} R = C(\vstarl, \vstarc) R$ (geometric series), which is large enough to offset the  cumulative efficiency penalty during subsequent catch-up.

\subsubsection*{Mandatory lane-changing (Route-following)}
Route-following is also achieved using a reward term. As described in Section \ref{sec:formulation}, the vehicle is provided a fixed route consisting of sections $s_1, \dots, s_n$, with each section $s_i$ consisting of several lanes $\ell_{i1}, \dots, \ell_{im}$. 

Let $s_i$ and $\ell_{ij}$ denote the current section and lane, respectively. Let $\Delta(\ell_{ij})$ denote the smallest number of lane changes that are necessary in order to reach an on-route lane from $\ell_{ij}$. Let $D(s)$ denote the current longitudinal distance to the end of section $s_i$. Then we define the immediate penalty from the current state as
\[ \Rmlc(s, a, s') = -\frac{\Delta(\ell_{ij})}{1 + D(s)} \]
For example, if the agent's route takes an off-ramp from a freeway at the end of its current section, but the agent is not currently in the exiting lane, then the agent gets a penalty inversely proportional to the distance to the exit and proportional to the number of lane changes necessary to reach the off-ramp lane. Please see Figure~\ref{fig:routing-penalty} for a visualization of the the penalties $\Rmlc$ for both routes in the freeway off-ramp scenario.

\begin{figure*}[!t]
\centering
\begin{minipage}{0.4\textwidth}        
    \includegraphics[width=0.75\textwidth]{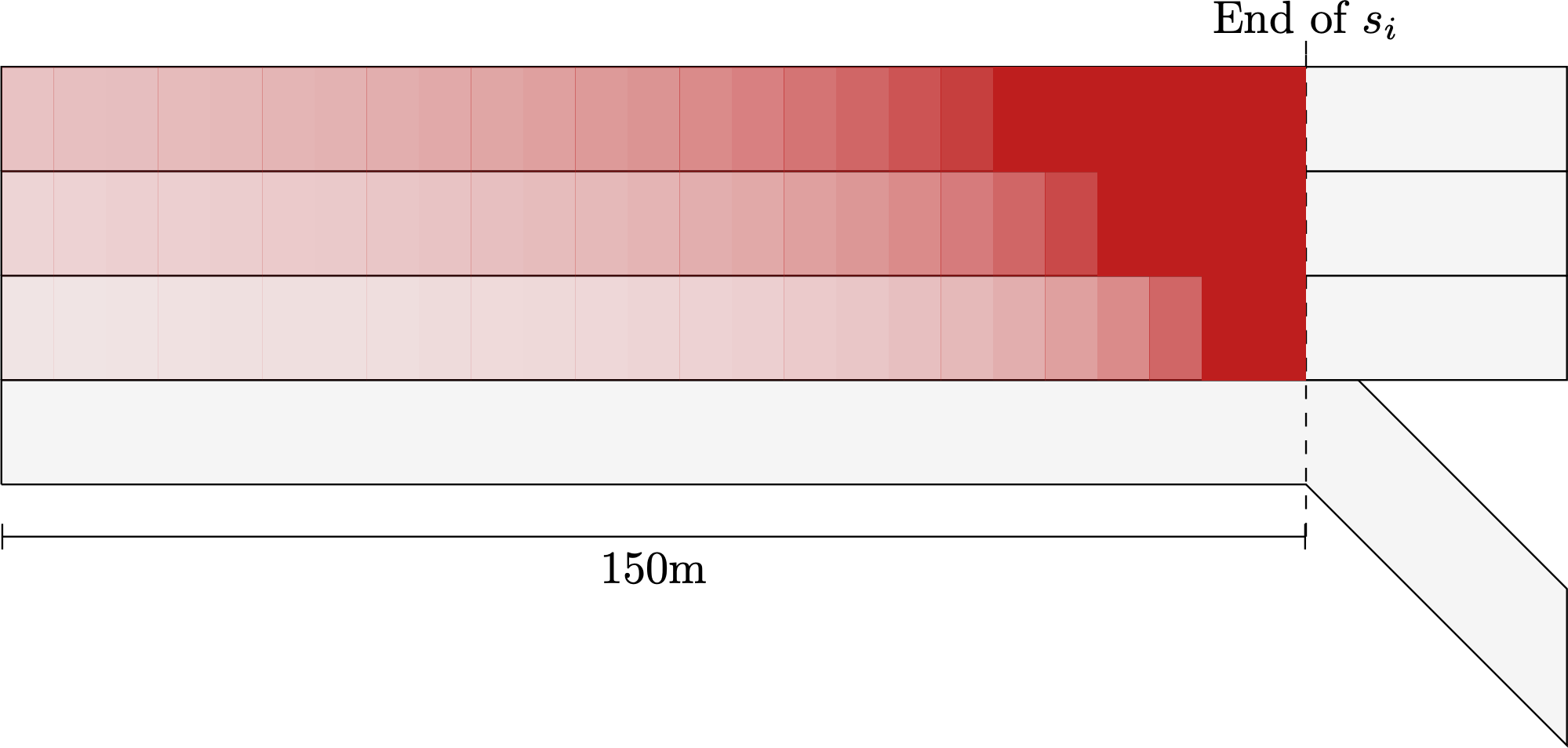}
\end{minipage}
\begin{minipage}{0.4\textwidth}
\includegraphics[width=0.75\textwidth]{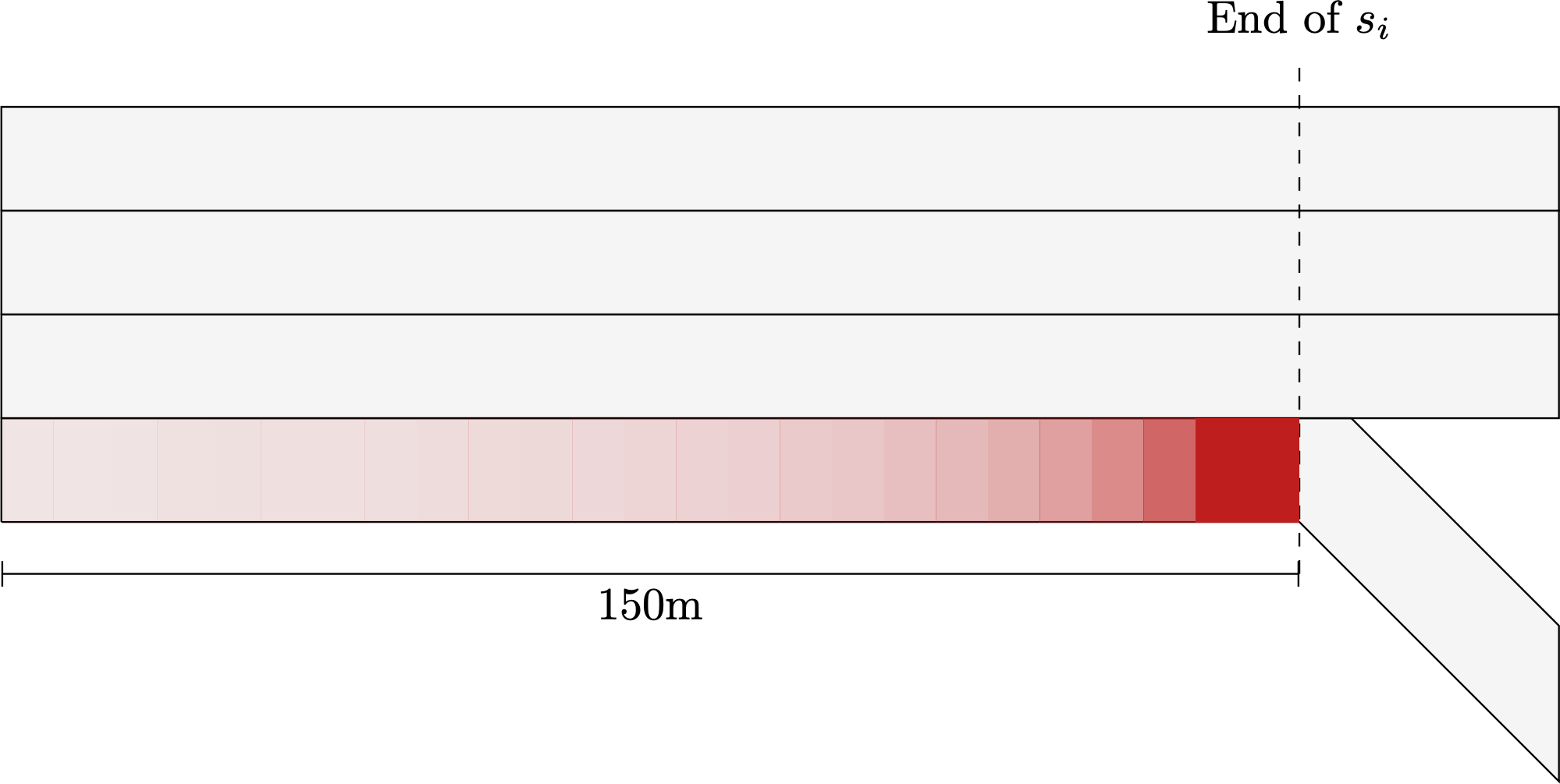}
\end{minipage}
\begin{minipage}{0.1\textwidth}
\includegraphics[height=100px]{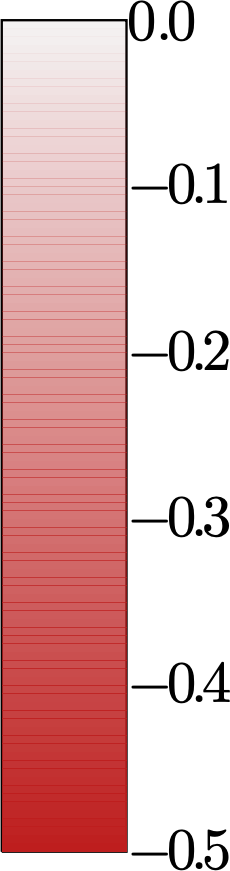}
\end{minipage}
    \caption{The mandatory lane-change penalty $\Rmlc$ accrued by a vehicle in a freeway off-ramp scenario, based on the vehicle's current position and route. For clarity of the illustration, the values are capped below at $-0.5$. Left: The vehicle route exits the freeway at the off-ramp. Right: The vehicle route remains on the freeway.}
    \label{fig:routing-penalty}
\end{figure*}

\subsubsection*{Combined optimization objective}

Putting the four terms together, the reward is then given by
\begin{equation} \label{eq:reward}
\begin{aligned}
    R(s,a,s') = & \Reff(s,a,s') \\
    & + \wcomf \cdot \Rcomf(s,a,s') \\
    & + \wlc \cdot \Rlc(s,a,s') \\
    & + \wmlc \cdot \Rmlc(s,a,s')
\end{aligned}
\end{equation}
where $\wcomf \geq 0,\, \wlc \geq 0 ,\, \wmlc \geq 0$ are the reward weights (hyperparameters).

The overall optimization objective is to find the optimal feasible policy
\[ \pi^* = \operatorname{argmin}_{\pi \in \Pi_{C,w}} J(\pi) \]
where $J(\pi)$ is the expected cumulative discounted reward, as defined in Equation~\ref{eq:J}.

\paragraph*{Normalization} The normalization factors in the reward terms $\Rcomf$ (division by $(a_E+\bE)/r$) and $\Rlc$ (division by $\vstarc$) are not logically necessary. The normalization terms bring the components of the reward to similar scales, making the choice of weights in (\ref{eq:reward}) simpler. In addition, it is standard practice in Deep RL to scale the reward terms to an interval such as $[-1, 1]$ or $[-10, 10]$ in order to aid neural network training.

\paragraph*{Unification of mandatory and discretionary lane-changing}
The composite objective function described in Equation~\ref{eq:reward} is a unified approach to optimizing both discretionary and mandatory lane change decisions. To the best of our knowledge, ours is the first work in the application of Deep RL to autonomous driving where both types of lane-change are treated in a unified manner.

\subsection*{Costs and weights} The hard constraints derived in Section~\ref{sec:formulation} may all be formulated as statewise constraints using cost-functions and weights. For example, the acceleration constraint (\ref{eq:acc-constraint}) may be rewritten in the form $\alpha - \asafec(s) \leq 0$, and the lane-change constraints (\ref{eq:lc-constraint}) in the form
\[
\resizebox{.9\hsize}{!}{%
    $\displaystyle
        I_{\{\ell = \mathrm{left}\}}(a) \left( v_E(s)r + \frac{v_E^2(s)}{2\bE} - \frac{v^2_{\Ll}(s)}{2\bLl} + \epsilon - \gl(s) \right) \leq 0
    $
}
\]
where $I_S(x)$ denotes the indicator function of the set $S$ ($I_S(x) = 1$ if $x \in S$ and 0 otherwise), and $(\alpha,\ell)$ are the two components of the action $a$.

\subsection{Learning an Optimal Policy}\label{sec:Learning}
The controller is trained using the Deep Deterministic Policy Gradient algorithm (DDPG) \cite{ddpg} with a continuous action-space encoding.

\subsubsection*{Action space encoding}
To apply DDPG to our statewise-constrained and mixed discrete-continuous problem, actions are encoded as pairs $(x,y) \in [-3,3]^2$. 

The coordinate $x$ is used to obtain the acceleration action as follows. Combining the safety acceleration constraint (\ref{eq:acc-constraint}) with the upper and lower bounds on acceleration, we obtain the acceleration upper bound $a_\mathrm{ub}(s) = \clip{\asafec(s)}{-\bE}{a_E}$, where the  notation $\clip{a}{b}{c} = \min(\max(a,b), c)$ is used. The coordinate $x$ is affinely projected onto the interval $[-\bE, a_\mathrm{ub}(s)]$ by $x \mapsto -\bE + \frac{x+3}{6}(a_\mathrm{ub}(s)+\bE)$. The resulting acceleration action satisfies all required constraints.

The coordinate $y$ is projected onto the lane-change action space component  by
\begin{equation*}
    y \mapsto \begin{cases}
        \mathrm{left}, & -3 \leq y < -1 \\
        \mathrm{cur}, & -1 \leq y < 1 \\
        \mathrm{right}, & 1 \leq y \leq 3
    \end{cases}
\end{equation*}
Furthermore, if the lane-change action can't be applied (for example, a left lane-change on the left edge of the road), the vehicle stays on its current lane.

\subsubsection*{DDPG} DDPG is a model-free, off-policy actor-critic algorithm. DDPG is an analogue of the DQN algorithm \cite{dqn} for  continuous action spaces.

Similarly to DQN, DDPG iteratively updates an approximation of the state-action value function $Q^*(s,a)$ of an optimal policy. The approximate function is often represented by a neural network with weights $\theta$, and denoted by $Q_\theta(s,a)$.

However, finding $\argmax_a Q_\theta(s,a)$ may itself present a difficult problem when the action space is continuous. Unlike DQN, in DDPG $Q_\theta$ is not used to directly obtain a policy. Instead, a deterministic policy $\pi_\phi$ is trained ($\pi$ is also often represented by a neural network with weights $\phi$) to maximize $Q_\theta(s,\pi_\phi(s))$ from each state.

Because of the roles the two functions play in training, $Q_\theta$ is called the critic and $\pi_\phi$ the actor.

The training loop proceeds as follows. For each environment step, the agent's action is obtained from the deterministic policy $\pi_\phi$ perturbed by an Ornstein-Uhlenbeck noise process for exploration. The experience tuple $(s,a,r,s')$ is stored in a replay buffer. Then, keeping $\phi$ fixed, the network $Q_\theta$ is updated using minibatch stochastic gradient descent to minimize the loss function
\[ L(\theta) = \frac{1}{\left| B\right|} \sum_{(s,a,r,s') \in B} \frac{1}{2} \left( r + \gamma Q_\theta(s',\pi_\phi(s')) - Q_\theta(s,a) \right)^2 \]
where $B$ is a minibatch of experience tuples sampled from the replay buffer. The update target $r + \gamma Q_\theta(s', \pi_\phi(s'))$ is motivated by the Bellman optimality equation as in $Q$-learning, but with the action decided by the actor network instead of being the action $a'$ that maximizes $Q_\theta(s',a')$ (such $a'$ may again be difficult to compute with a continuous action space). 

Then, keeping $\theta$ fixed, the network $\pi_\phi$ is updated using minibatch stochastic gradient ascent to maximize the expected reward
\[ J(\phi) = \frac{1}{\left| B\right|} \sum_{(s,a,r,s') \in B} Q_\theta(s, \pi_\phi(s)) \]

For greater training stability, target networks are introduced in the above procedure. For more details on DDPG, please see the original paper \cite{ddpg}.

As demonstrated in Figure~\ref{frame}, our main framework is a hybrid learning framework. Based on the current traffic environment, the maximal safe target speed is calculated. The RL vehicle will consider the target speed and decide on the longitudinal acceleration value and the lateral decision on whether to switch left, right, or stay in the main lane. Based on feedback from the environment, the RL algorithm will learn the policy.

\afterpage{
\begin{figure}[ht]
\vskip -0.1in
\begin{center}
\centerline{\includegraphics[width=8cm]{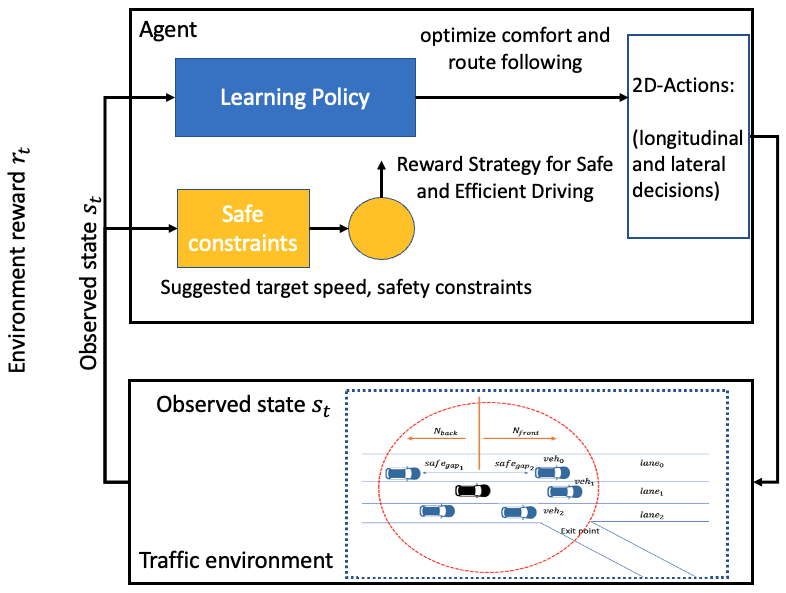}}
\caption{Visualization of the overall framework. Unlike previous work~\cite{cao2020highway}, we enforce safety through action constraints instead of through a safety term in the reward function. The other objectives are optimized subject to satisfying the safety constraints.} 
\label{frame}
\end{center}
\vspace{-0.5cm}
\end{figure}
}

\section{\method~Training setup}
\subsection{Experiment environments}
Training and evaluation are carried out in the SUMO microsimulator using the TraCI Python API \cite{sumo}. The  simulation step-length of 0.1s is used, which is also the controlled vehicle's reaction time $r$.

Experiments are performed in two different environments: a simple loop network, and a network based on a real-world freeway.

\subsubsection{Loop network}
\label{sec:loop-network}
The loop network of Figure~\ref{fig:loop} is used as one of the experiment environments. The effect of the curvature on vehicle speeds is disabled, which makes the network effectively an infinite freeway (without on-ramps and off-ramps). The loop network is used to evaluate efficiency, comfort and discretionary lane-changing in isolation from route-following.

\subsubsection{Interchange of Queen Elizabeth Way (QEW) and Erin Mills Parkway / Southdown Road}
\label{sec:qew}
We created a network based on real road geometry, the Queen Elizabeth Way (QEW) interchange with Erin Mills Parkway (on the North side) and Southdown Road (on the South side) near Toronto, Canada. The network is shown in Figure~\ref{fig:qew-sumo} (in SUMO) and Figure~\ref{fig:qew-google-maps} (in Google Maps). The off-ramp and on-ramp to the QEW are used to test route-following tasks (in addition to optimization of efficiency and comfort, subject to safety constraints, as before).

\begin{figure*}[t]
    \centering
    \hfill
    \begin{minipage}[t]{0.3\textwidth}
        \centering
        \includegraphics[width=\textwidth]{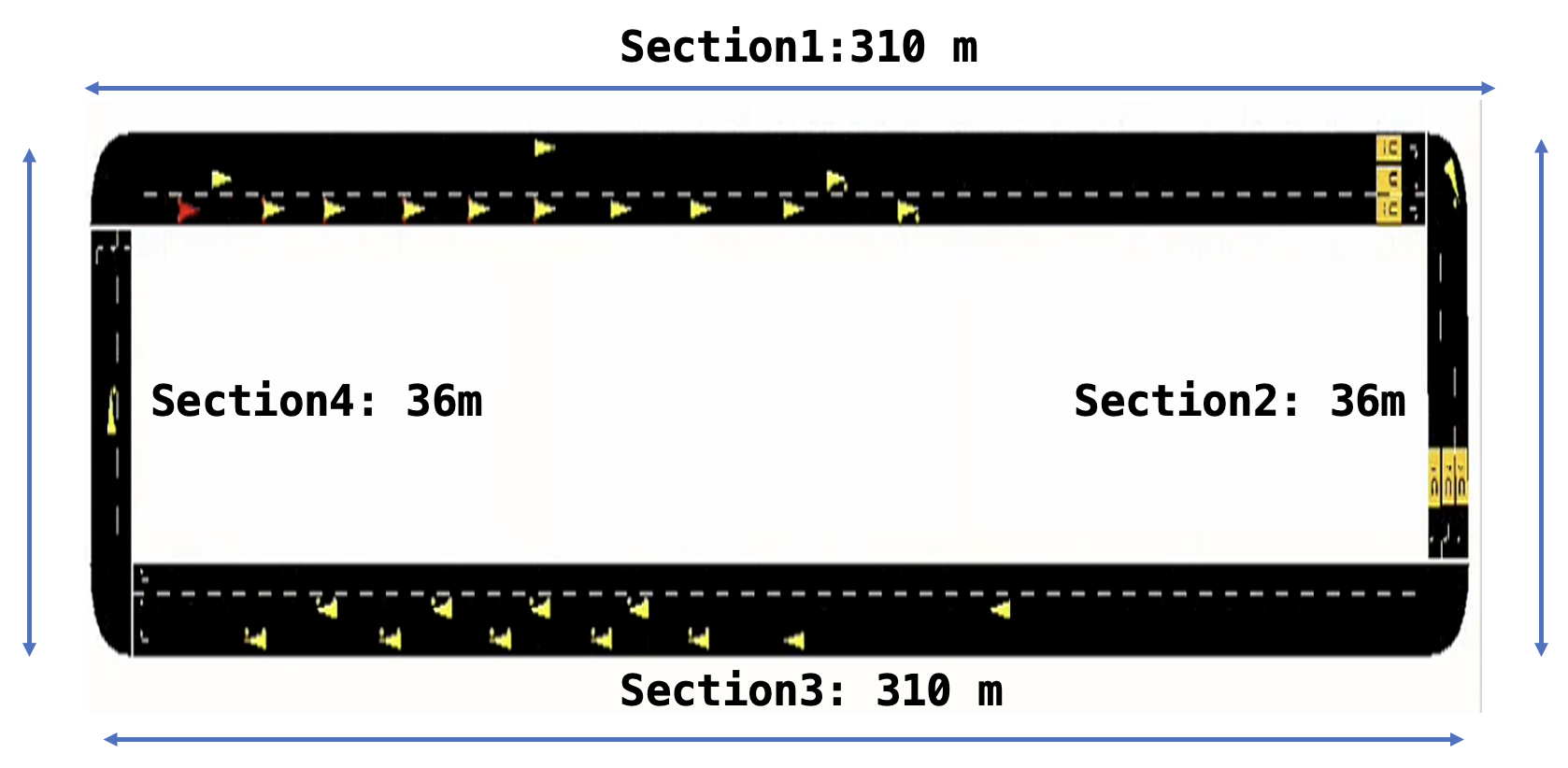}
        \caption{Geometry of the loop network. The effect of curvature on vehicle speed has been disabled.}
        \label{fig:loop}
    \end{minipage}
    \hfill
    \begin{minipage}[t]{0.25\textwidth}
        \centering
        \includegraphics[width=\textwidth, trim={0 0 0 3cm},clip]{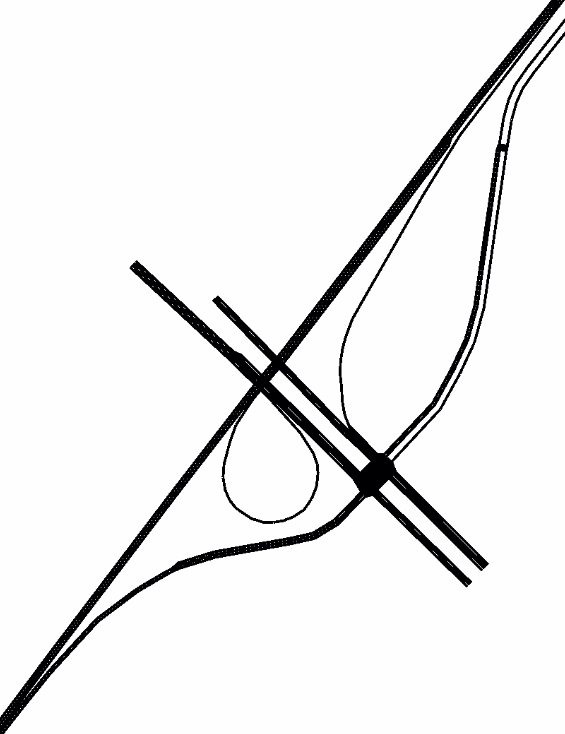}
        \caption{Geometry of the interchange of Queen Elizabeth Way (QEW) and Erin Mills Parkway / Southdown Road (SUMO).}
        \label{fig:qew-sumo}
    \end{minipage}
    \hfill
    \begin{minipage}[t]{0.3\textwidth}
        \centering
        \includegraphics[width=\textwidth]{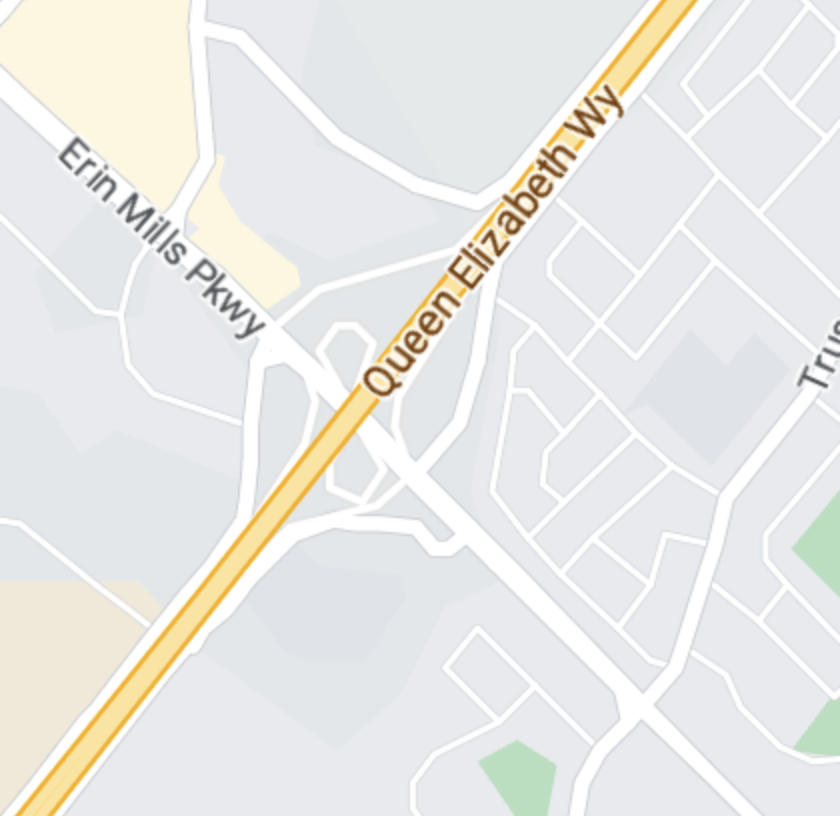}
        \caption{Geometry of the interchange of Queen Elizabeth Way (QEW) and Erin Mills Parkway / Southdown Road (Google Maps)}
        \label{fig:qew-google-maps}
    \end{minipage}
    \hfill
\end{figure*}

\subsection{Modelling uncontrolled vehicles}
The uncontrolled vehicles are modelled using the ``Krauss'' car-following model (SUMO default) and the ``SL2015'' lane-change model cf. \cite{sumo}. This configuration may model either human drivers or other AVs. For brevity, we sometimes refer to the uncontrolled vehicles as human-driver  vehicles (HVs).

\subsection{Training scenarios}
To help generalization to varied or unseen  scenarios during testing, the target speeds (speed limits) and initial speeds of the human-driver vehicles are varied episode-to-episode.

\subsubsection{Loop network}
\label{sec:loop-training}
The loop network is a lightweight network, which allows for faster training and iteration. Each episode is 5,000 time-steps long, to gather sufficient experience for learning.

The controlled vehicle is trained with 25 human-driver vehicles. The human-driver vehicles are injected from a designated starting section with random departure speeds. 
Each episode, the human-driver vehicles' speed limit is resampled uniformly at random from $\{10,20,22,25\}$, and the controlled vehicle's speed limit from $\{16, 22, 28, 36\}$. The set of controlled vehicle speeds is higher than the set of human-driver vehicle speeds to aid  the controlled vehicles in learning to bypass the slower human-driver vehicles, in order to achieve higher efficiency. 

In training, all reward weight hyperparameters are set to one: $\wcomf = \wlc = \wmlc = 1$.

\subsubsection{QEW}
\label{sec:qew-training}
In the QEW network, the simulation horizon is 6,000 steps, to allow sufficiently many environment steps for the RL agent to leave the network.

The uncontrolled vehicles are generated at a fixed inflow rate at the Southern boundary of the freeway mainline. The inflow rates are uniformly  sampled between 500 veh/hour and 1000 veh/hour. The controlled vehicle is generated at the Southern boundary of the freeway mainline after a warm-up period. The speed limits for the controlled vehicle and uncontrolled vehicles are sampled from the same sets as for the loop network.

The controlled vehicle has two possible routes: to stay in the freeway or take the off-ramp. Each episode, we assign the route to the vehicle randomly, each with probability $1/2$.

As in the loop network, in training, all reward weights are set to one $\wcomf = \wlc = \wmlc = 1$.

\subsection{DDPG parameters}
The configuration of the DDPG model is summarized in Table~\ref{ddpgtable}. Five random seeds were sampled, and training was run for 1000 episodes for each seed. The checkpoint that achieved the highest cumulative reward across the 5 seeds  was selected for evaluation.

\begin{table}[t]
\centering
\caption{DDPG and simulation Configuration}
\resizebox{.9\linewidth}{!}{%
\begin{tabular}{lc}
\hline\hline \noalign{\smallskip}
Parameter & Value \\
\hline
Discount factor ($\gamma$) & 0.99 \\
Simulation time step  & 0.1 second \\

Maximum Episode horizon  & \begin{tabular}{c}5,000 (Loop) \\ 6,000 (QEW) \end{tabular} \\
Critic $Q_\theta$ learning rate & 3e-4 \\
Actor $\pi_\phi$ learning rate & 3e-4 \\
Target soft update rate & 0.005 \\
Warm-up  & 1,000 steps \\
Replay buffer size & 1e+6 \\
Actor and critic MLP NN number of layers & 3 \\
Actor and critic NN number of nodes in hidden layer & 256 \\
NN Activation functions & torch.relu \\
NN Last activation function & torch.tanh \\
Batch size for gradient ascent & 64 \\
\hline\hline \noalign{\smallskip}
\end{tabular}
}
\label{ddpgtable}
\end{table}

\subsection{Convergence}
Figure~\ref{fig:reward_convergence} shows the convergence of the efficiency and comfort rewards during training in the QEW network. The solid curves in the figure display the averages of the two rewards during training over 5 different training runs with different random seeds, and the shaded area represents values within one standard deviation across the 5 training runs. The scores initially improve rapidly, and converge to stable values, with slight peaks around episode 400. The training in the loop network had similar dynamics.

In the remaining sections, we shall analyze the trained policy and compare it against several baselines.

\begin{figure}[t]
\vskip -0.1in
\centerline{\includegraphics[width=0.8\linewidth]{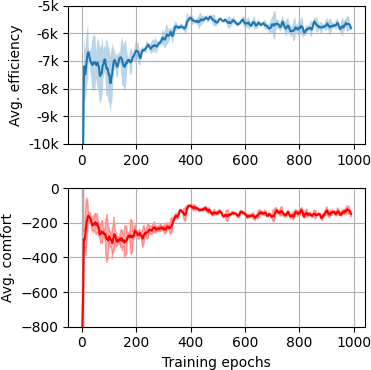}}
\caption{Convergence of the Efficiency (top) and Comfort (bottom) rewards during training.Initial few epochs, the agent adopts a conservative policy, often remaining stationary, resulting in low efficiency scores. Over time, however, it learns to balance safety and efficiency, achieving high levels of driving comfort.  } 
\label{fig:reward_convergence}
\vspace{-0.5cm}
\end{figure}

\section{Long-Term Longitudinal Steady-State, and Platooning}
\label{sec:longitudinal-analysis}
In this section, the long-term longitudinal behaviour of a collection of \method~vehicles is analyzed. We begin by studying the long-term longitudinal behaviour of a collection of vehicles that drive at their maximal safe speed. For a single vehicle with a leader at a constant speed, detailed results are obtained. Then, qualitative implications for collections of \method~vehicles are discussed. Our conclusion is that, speed-limit permitting, in the long-term the vehicles form platoon driving at the same speed as the lead vehicle, with fixed gaps between consecutive vehicles that are determined by the leader speed and the vehicle's reaction time.

\subsection{Single Follower with a Constant-Speed Leader}
Consider a two-vehicle system, where the leader vehicle drives at a constant speed $w$, and the follower always drives at its maximal safe speed. In addition, to simplify the analysis, we assume that the speed limit is infinite, there is no acceleration limit, and that there is a single lane. We discuss relaxations of these assumptions at the end of this section.

The motion of the follower can be fully described by the two-dimensional discrete-time dynamical system 
    \[(g_t, v_t)\]
where $g_t$ is the gap between the follower and the leader at time-step $t$, and $v_t$ is the velocity of the follower at time-step $t$, with update rules
\begin{equation*}
    \resizebox{.9\hsize}{!}{%
    $\displaystyle
    \begin{split}
    g_{t+1} &= g_t - \dt(v_t - w) \\
    v_{t+1} &= -\frac{r\bE}{2} + \sqrt{\left(\frac{r\bE}{2}\right)^2 -2\bE \left( \frac{r v_t}{2} - \frac{w^2}{2\bLc} - g_t + \epsilon \right)}
    \end{split}
    $}
\end{equation*}
where the velocity update is obtained from Equation~\ref{eq:safe-speed} for $\vsafec(t+1)$.

To begin the analysis, it is helpful to notice that Equation~\ref{eq:safe-speed} may be rearranged to
\begin{equation}
\label{eq:safe-speed-rearranged}
\resizebox{.9\hsize}{!}{%
$\vsafec(t+1) = -\frac{r\bE}{2} + 
\sqrt{\left(v_{\Lc}(t) + \frac{r\bE}{2}\right)^2  + 2\bE\, \delta g(t) - \bE r\, \delta v(t)}$}
\end{equation}
where
\begin{equation}
    \label{eq:deltas}
    \begin{split}
    \delta g(t) &= \gc(t) - v_{\Lc}(t)r - \frac{\bLc - \bE}{2\bLc\bE} v_{\Lc}(t)^2 - \epsilon \\
    \delta v(t) &= v_E(t) - v_{\Lc}(t)
    \end{split}
\end{equation}
This rearrangement makes the steady-state equilibrium target more transparent.
\begin{thm}
The state $\displaystyle (g^*, v^*) = \left(w\dt + \frac{\bLc - \bE}{2\bLc\bE}w^2 + \epsilon,\ w\right)$ is a system equilibrium.
\end{thm}
\begin{proof} Indeed, assuming that $(g_t, v_t) = (g^*, v^*)$, we compute that
\begin{equation*}
    \resizebox{.9\hsize}{!}{%
    $\displaystyle
    \begin{split}
        g_{t+1} &= g_t - \dt(0) = g_t \\
        v_{t+1} &= -\frac{\bE\dt}{2} + \sqrt{\left( w + \frac{\bE\dt}{2}\right)^2 + 2\bE(0) - \bE\dt(0)}= w = v_t
    \end{split}$
    }
\end{equation*}
So that $(g_{t+1}, v_{t+1}) = (g_t, v_t) = (g^*, v^*)$.
\end{proof}
\begin{thm}
For $w > 0$, the equilibrium point $(g^*, v^*)$ is asymptotically stable. That is, there exists an open set $U \subseteq \mathbb{R}^2$ containing $(g^*, v^*)$ such that for all initial points $(g_0, v_0)$ in $U$, 
    \[ \lim_{t\rightarrow \infty} (g_t, v_t) = (g^*, v^*). \]
\end{thm}
\begin{proof}To demonstrate asymptotic stability, it is sufficient to look at the linearization of the system about the equilibrium point $(g^*, v^*)$ and show that both eigenvalues of the Jacobian matrix have norm $< 1$ \cite[Th.~4.7]{discdynsys}.

The Jacobian of the update functions at the point $(g,v)$ is
\begin{align*}
    \begin{pmatrix}
        1 & -r \\
        \frac{2\bE}{2\sqrt{A(g_t,v_t)}} &  \frac{-\bE r}{2\sqrt{A(g_t,v_t)}}
    \end{pmatrix}
\end{align*}
where $A(g_t, v_t) = \left( w + \frac{\bE\dt}{2}\right)^2 + 2\bE\,\delta g(t) - \bE\dt\,\delta v(t)$.

Therefore, at the equilibrium point $(g^*, v^*)$, we have the Jacobian
\begin{align*}
    \begin{pmatrix}
        1 & -r \\
        \frac{2\bE}{2w + \bE r} &  \frac{-\bE r}{2w + \bE r}
    \end{pmatrix}
\end{align*}
The eigenvalues of the Jacobian at $(g^*, v^*)$ are 
\[
    \lambda_{\pm} = \frac{w \pm \sqrt{w^2 - 2w\bE\dt - \bE^2\dt^2}}{2w + \bE r}
\]
It is elementary to verify that $\left| \lambda_\pm \right| < 1$ whenever $w > 0$ (please see Appendix \ref{app:jac-eigen-detail}). 
\end{proof}
\begin{remark}
For the $w = 0$ case, we have complex eigenvalues $\lambda_\pm = \pm \sqrt{-1}$, both of which have norm 1. Therefore, the central eigenspace spans the entire space and we cannot conclude whether the equilibrium point is stable or not from general theory \cite[p.~124]{dynsys}. Empirically, however, we find that the $w = 0$ case behaves similarly to the $w > 0$ case.
\end{remark}
%The general theory gives us no description of the open set $U$. Empirically, however, we find that the equilibrium is attracting for a broad range of initial conditions.

\subsubsection*{Removal of simplifying assumptions} We  briefly discuss how removing the simplifying assumptions made at the start affects the analysis. With a bounded maximal acceleration, the follower may need an initial speed-up period in order to reach a high enough speed to catch up with the leader, after which the analysis reduces to the previous case. With a speed-limit, the analysis proceeds as before as long as $w$ is lower than the speed limit. With $w$ equal to the speed limit, the follower will not be able to catch up to the leader.

\subsection{Qualitative implications for \method~platoon formation}
In longitudinal motion, \method~is rewarded for taking the action that brings its velocity as close to $\vsafec(t+1)$ as possible, subject to minimizing the discomfort penalty integrated over its trajectory (and acceleration/speed-limit constraints).

Thus, although this is no longer subject to proof, our expectation is that a well-trained \method~model following a constant-speed leader will have similar long-term equilibrium properties to the system analyzed in the previous section, that is, will drive at speed $w$ with gap $w\dt + \frac{\bLc - \bE}{2\bLc \bE} w^2 + \epsilon$, and that this equilibrium will be stable with respect to small perturbations in the initial conditions.

For a platoon of \method~vehicles following a fixed-speed leader, we proceed to argue inductively. As vehicle $n$ approaches its equilibrium with its respective leader, the speed of vehicle $n$ will continue to converge to the common platoon speed $w$. Therefore, vehicle $n$ will start to act as a constant-speed (up to a small error) leader for vehicle $n+1$. We may then apply our analysis to the vehicle $(n, n+1)$ system, concluding as before that in steady-state vehicle $n+1$ will have gap and velocity described by the previously obtained formulas. 

Thus, we expect that a well-trained \method~system will form a platoon with a common velocity $w$ and inter-vehicle gaps given by
\[w\dt + \frac{\bLc - \bE}{2\bLc \bE} w^2 + \epsilon\] 
This expectation is confirmed by empirical tests. An example of a four-vehicle platoon (constant-speed leader and three \method~followers) on a single-lane loop  reaching its steady-state is displayed in Figure~\ref{fig:platoon-formation}.

A closer analysis of the string-stability properties of the resulting platoon will be the subject of future work.

Finally, we note that if in addition the vehicle velocities are constrained by a speed limit, it may not be possible for a follower vehicle to catch up to its leader (if the leader drives at the speed limit). Therefore, with a speed-limit a collection of \method~vehicles is expected to form into several platoons, grouped by whether or not the follower is able to catch the leader under the speed-limit constraints.

\begin{figure*}[t]
\centering
{\includegraphics[width=0.8\textwidth]{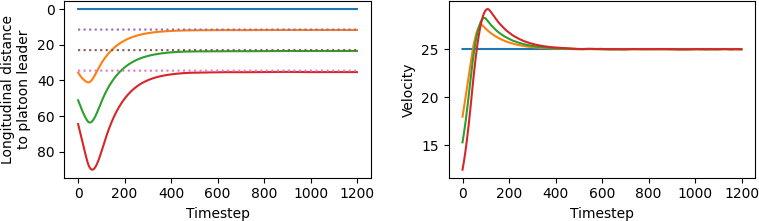}}
\caption{Convergence to steady-state of a platoon consisting of a leader (blue) driving with a constant speed of 25 m/s and three \method~follower vehicles (orange, green, red). The minimal gap setting is $\epsilon=4$ m. Each vehicle is 5 m long. Left: Relative timespace diagram (longitudinal distance of each vehicle from the leader as a function of time). Right: Vehicle velocities. The steady-state distances match our theoretically predicted values (dotted lines).} 
\label{fig:platoon-formation}
\end{figure*}

\begin{table*}[t]
    \centering
    \caption{Loop network statistics comparison. 5000 simulation steps with 30 random seeds}
    \resizebox{0.95\textwidth}{!}{%
    \begin{tabular}{|l|ccc|ccc|ccc|}\hline%\noalign{\smallskip}
         \multirow{2}{*}{Method} & \multicolumn{3}{c|}{Normal driving} & \multicolumn{3}{c|}{Heavy traffic} & \multicolumn{3}{c|}{Emergency braking} \\% 
         \cline{2-10}
         & Speed (m/s) & Jerk (m/s$^3$) & Crash rate & Speed (m/s) & Jerk (m/s$^3$) & Crash rate & Speed (m/s) & Jerk (m/s$^3$) & Crash rate\\
         \hline
        IDM+MOBIL & 27.48 & 0.87 & 0\%
                  & 25.63 & 0.04 & 0\%
                  & 26.76 & 0.99 & 0\% \\
        Gipps+Greedy  &32.51 &1.39 & 0\%
                      &26.59 &0.09  & 0\%
                      &30.51 &1.47  & 0\%\\
         PPO-based RL\cite{ye2020automated}  &28.84 &1.31 &3\%
         &25.88 &0.03   &7\% 
         &28.11 &1.20  &21\% \\
        \method  &32.45 &\textbf{0.86}   & \textbf{0\%}
                 &26.32&\textbf{0.02} & \textbf{0\%}
                 &\textbf{31.73} &\textbf{0.89}   &  \textbf{0\%}\\
        %\noalign{\smallskip}
        \hline
    \end{tabular}
    }
    \label{table:loop-results}
\end{table*}

\section{Baselines}
In the rest of this work, \method~is evaluated in several experiments, and compared with the following baselines:
\begin{itemize}[wide]
    \item \textbf{IDM with MOBIL:} The ``Intelligent Driver Model'' (IDM) \cite{idm} car-following with ``Minimizing Overall Breaking Induced by Lane-changes'' (MOBIL) \cite{treiber2016mobil} lane-changing is a common model of realistic driving (which can be used to model either human drivers or AVs).

    \item \textbf{PPO-based lane-change maneuver (Ye et al. ~\cite{ye2020automated}) :}  a good RL baseline for the discretionary lane-change comparison. 

   \item \textbf{Highway exiting planner (Cao et 
    al. ~\cite{cao2020highway}): }
    The deep reinforcement learning-enhanced highway exiting planner~\cite{cao2020highway} is another baseline. The RL framework is based on Monte Carlo Tree Search (MCTS) with safety enhancements based on traffic rule constraints. This baseline is used to evaluate the performance of our agent's route-following. This is a good baseline for mandatory lane change comparison.

   \item \textbf{Gipps with greedy lane selection:}
   In the case when the ego vehicle's speed is constrained by a leader vehicle (as opposed to the free-driving regime in which the speed is constrained by the road's speed limit), the Gipps car-following model \cite{gipps} is also based on the Vienna safety convention, similar to our previous longitudinal car-following work SECRM \cite{secrm}. In the terminology of this paper, the Gipps model always drives at the maximal safe velocity $\vsafec$ defined in Equation \ref{eq:safe-speed}, disregarding comfort. In \cite{secrm}, we noticed that SECRM converges to Gipps if the comfort weight is set to 0.  (It should be noted that in the case when the ego vehicle's speed is constrained by the speed limit, Gipps switches to a  model fit to experimental human-driver data, and is  suboptimal compared to \cite{secrm}.) 
   
  We extend Gipps with a lane-changing model that switches to one of the adjacent lanes whenever the difference in target speeds with the current lane is larger than a threshold (typically, the threshold is 3 m/s), subject to similar lane-change safety constraints as \method. The resulting model is called Gipps with greedy lane selection (Gipps+Greedy). We conjecture that \method~should converge to Gipps with greedy lane selection in the case that the comfort weight $\wcomf$ is set to 0. Gipps with greedy lane selection is included as an extreme baseline that optimizes efficiency while disregarding comfort.

\end{itemize}
For the mandatory lane change decisions for the baselines, we use the rule-based method from SUMO~\cite{erdmann2015sumo}.

\section{Experiments on Discretionary Lane-Changing  in the Loop Network}
In this section, the controller  trained in the loop network as described in Section~\ref{sec:loop-training} is evaluated. Three test scenarios are created for the loop network,:
\begin{itemize}[wide]
    \item \textbf{Normal driving.} The controller is tested with 25 human-driver  vehicles. The human-driver vehicles have the same type of driving model as in training. The speed-limit for the human-driver vehicles is chosen to be 17 m/s, and the speed-limit for the controlled vehicle is 34 m/s. We are careful to select speed-limit values that cannot be sampled during training, so that the controller must generalize from train to test.

    \item \textbf{Congested traffic.} The heavy traffic scenario is similar to the previous regular driving scenario, but the number of human-driver vehicles is doubled to 50 to evaluate how well the controlled vehicle can handle congested scenarios. 

    \item \textbf{Emergency braking.} To measure the controlled vehicle's safety performance, in the emergency braking scenario the human-driver vehicles will undergo emergency braking in a specific section with maximum deceleration until reaching minimum speed of 3 m/s. We applied emergency braking in different sections of the loop network in the testing phase in order to evaluate the generalization of the controller. 
\end{itemize}

\subsection*{Results analysis}
Table~\ref{table:loop-results} compares the average speed, jerk and crash rate of \method~with three baselines in the three test scenarios described above.

Overall, \method~is comparable with the best performers in both efficiency (average speed) and comfort (average jerk), and has 0\% crash rate. However, we note the other best performer in either speed or jerk scores lower on the other metric, making \method~the best overall controller. More specifically, \method~and Gipps+Greedy reach the best average speed, but \method~has significantly lower jerk. IDM and \method~have comparable jerk, but \method~has superior efficiency (higher average speed).

The results match intuition: we expect \method~to converge to Gipps+Greedy in the limit as $\wcomf \rightarrow 0$. \method~trades in some potential loss of efficiency in exchange for highly improved comfort. On the other hand, IDM is smoother, but is either unsafe or inefficient because of its fixed (non-dynamic) target gap (please refer to our previous work \cite{secrm} for further comparison with the IDM non-dynamic gap).

%We also observe some interesting effects in the three scenarios. Firstly, we observe that in light traffic,  the lane-changes make some shockwaves sometimes which is because \method~is selfishly optimizing its safety, efficiency, and comfort, but doesn’t care about other vehicles. Introducing cooperative behaviour will be the subject of future work.  

In the congested traffic scenario, we observe fewest differences between the compared methods, as the vehicles are packed closely together, constraining movement and making overtaking difficult.

During testing, the \method~vehicle learned to bypass the slower vehicles in both normal driving and emergency stop scenarios, improving efficiency. The method is successful in learning effective discretionary lane changes.

We highlight that \method~achieves 0\% crash rate in all three scenarios. On the other hand, the compared RL baseline (PPO-based RL) has a 3--7\% crash rate in regular and congested scenarios and the very high 21\% crash rate in the emergency braking scenario. These results support our motivating claim that while RL is a promising approach for reaching autonomous driving, without stricter action constraints (such as ones put in place for \method), ensuring sufficient safety for deployment may be difficult.

\begin{table*}[t]
    \centering
    \caption{Bypassing scenario results with equal probability of route staying on freeway and route taking off-ramp.}%Random route choice with same initial speed both main lane and exit lane flow}
    \resizebox{\linewidth}{!}{%
        \begin{tabular}{@{} l *{13}{c} @{}}
        \toprule
        \multirow{2}{*}{Method} & \multicolumn{4}{c}{$H=5$} & \multicolumn{4}{c}{$H=10$} & \multicolumn{4}{c}{$H=20$} \\
        \cmidrule(lr){2-5} \cmidrule(lr){6-9} \cmidrule(lr){10-13}
        & Speed (m/s) & Jerk (m/s$^3$) & Crash Rate & Route Miss Rate & Speed (m/s) & Jerk (m/s$^3$) & Crash Rate & Route Miss Rate & Speed (m/s) & Jerk (m/s$^3$) & Crash Rate & Route Miss Rate \\
        \midrule
        IDM+MOBIL & 18.67 & 1.23 & 13\% & 31\% & 19.21 & 1.39 & 12\% & 24\% & 20.01 & 1.03 & 3\% & 4\% \\
        Gipps+Greedy & 20.34 & 1.86 & 14\% & 27\% & 23.98 & 2.01 & 15\% & 20\% & 23.47 & 1.32 & 5\% & 0\% \\
        PPO-based RL & 19.31 & 1.31 & 23\% & 33\% & 20.12 & 1.38 & 24\% & 27\% & 21.31 & 1.23 & 16\% & 19\% \\
        Mandatory Lane Change RL  & 18.21 & 1.71 & 18\% & 17\% & 19.12 & 1.58 & 17\% & 13\% & 19.98 & 1.31 & 5\% & 1\% \\
        \method & \textbf{20.38} & \textbf{1.08} & \textbf{0\%} & \textbf{15\%} & 23.08 & \textbf{1.13} &\textbf{ 0\%} & \textbf{10\%} & 23.42 & \textbf{1.09} & \textbf{0\%} & \textbf{0\%} \\
        \bottomrule
        \end{tabular}%
    }
    \label{table:qew-bypassing}
\end{table*}

\begin{table*}[t]
    \hfill
    \begin{minipage}{0.35\textwidth}
        \centering
        \caption{Emergency braking scenario results} %Route stays on freeway, lane change for leader and ego vehicle disabled}
        \resizebox{\textwidth}{!}{%
        \begin{tabular}{lcc}
            \toprule
            Method & Jerk (m/s$^3$) & Crash Rate \\
            \midrule
            IDM+MOBIL & 0.95 & 6\% \\
            Gipps+Greedy & 2.43 & 12\%  \\
            PPO-based RL & 1.08 & 24\%  \\
            Mandatory Lane Change RL & 1.91 & 17\%  \\
            \method & \textbf{0.94} & \textbf{0\%} \\
            \bottomrule
        \end{tabular}
        }
        \label{table:qew-emergency}
    \end{minipage}
    \hfill
    \begin{minipage}{0.6\textwidth}
        \centering
        \caption{Merging scenario results}
        \resizebox{\textwidth}{!}{%
        \begin{tabular}{lcccc}
        \toprule
        Method & Speed (m/s) & Jerk (m/s$^3$) & Crash Rate  & Merge Miss Rate \\ \midrule
        IDM+MOBIL & 18.67 & 1.23 & 13\% & 31\% \\
        Gipps+Greedy & 20.34 & 1.86 & 14\% & 27\% \\
        PPO-based RL & 19.31 & 1.31 & 23\% & 33\% \\
        Mandatory Lane Change RL  & 18.21 & 1.71 & 11\% & 17\% \\
        \method & \textbf{22.53} & \textbf{1.21} &
        \textbf{0\%} & \textbf{13\%} \\ \bottomrule
        \end{tabular}
        \label{table:qew-merging}
        }
    \end{minipage}
    \hfill
\end{table*}

\begin{table}[t]
    \centering
    \caption{Bypassing HW 5: Dependence on Reaction Time} % (max acc=3, max dec=5, ego AV’s target speed is 25 m/s, and other vehicles’ target speeds are 15 m/s)}
    \begin{tabular}{lccc}
        \toprule
        Method & Speed (m/s) & Jerk (m/s$^3$) & Crash Rate \\
        \midrule
       $r_E=r_F=0.1$ & 21.29 & 1.09 & 0\% \\
        $r_E=r_F=1$ & 19.21 & 1.08 & 0\% \\
       $r_E=0.1,\ r_F=1$ & 20.38 & 1.08 & 0\%  \\
        %$r=\rmodel=0.1, \rtrue=1$ & 22.31 & 1.11 & 9\% \\
        \bottomrule
    \end{tabular}
    \label{table:differentassumption}
\end{table}

\section{Experiments on Bypassing, Route-Following, and Merging in the QEW}
In this section, the controller trained in the QEW network as described in Section \ref{sec:qew-training} is evaluated.
We aim to investigate: (1) whether our agent can conduct safe, efficient and comfortable lane-changes; (2) whether our agent can follow the provided route, i.e., stay on main-lane or exit; (3) whether our agent guarantees safety including in emergency braking; (4) whether our agent can generalize to unseen scenario configurations.

To ensure that the test scenario's configuration is different from training to avoid over-fitting, we use 15 m/s for human-driver vehicles' speed limit and 25 m/s for the controlled vehicle's speed limit.

Three test scenarios are created for the QEW network:
\begin{itemize}[wide]
\item \textbf{Bypassing.}
The motivation for constructing the bypassing test environment is to evaluate the ability of the controlled  vehicle to conduct lane changes to bypass slower vehicles. The construction of the environment is based on the following rule as illustrated in Figure~\ref{zigzag}. The head-to-head distance between adjacent  human-driver vehicles is maintained at $H$. Lane-changing for human-driver vehicles is disabled. We set the speed of all human-driver vehicles as $v_{\mathrm{human}}$  m/s which is slower than the controlled vehicle's target speed of  $v_{\mathrm{AV}}$ 
 m/s. The controlled vehicle (in black) is initialized behind of the human-driver vehicles' zig-zag fleet. The controlled  vehicle's route is assigned to be either to stay on the freeway or exit using the off-ramp with equal probability of $1/2$.

\begin{figure}[ht]
\vskip -0.1in
\begin{center}
\centerline{\includegraphics[width=0.9\linewidth]{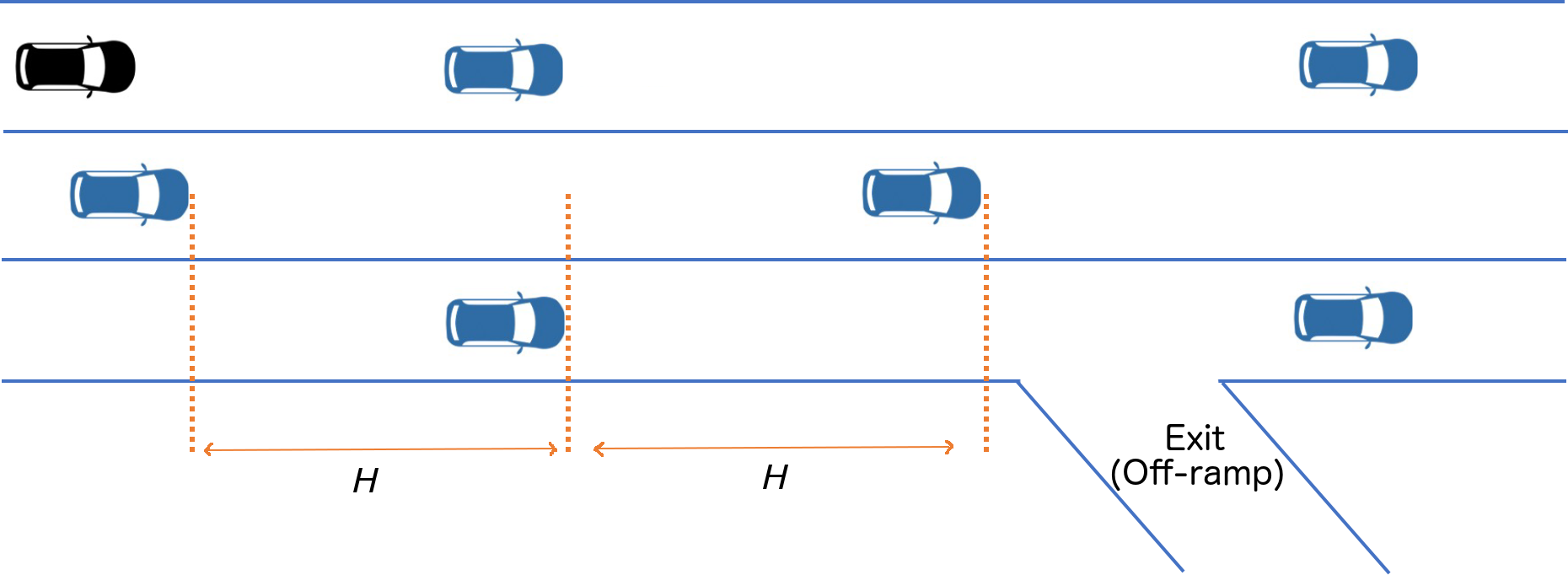}}
\caption{Diagram of the bypassing scenario in the QEW network.} 
\label{zigzag}
\end{center}
\vspace{-0.5cm}
\end{figure}

\item \textbf{Emergency braking.}
As shown in Figure~\ref{brakeqew}, to test the ability of \method~to handle emergency braking, another scenario is constructed. The headway distance between each human-driver vehicle $H$ is set to be  too small for the controlled vehicle to bypass, forcing the controlled vehicle to stay behind. Maximum deceleration of $-3$ m/s$^{2}$ is applied to each human-driver vehicle in the shaded area. In the emergency braking scenario in the loop network, the braking leader vehicle occupied only one of the lanes, which permitted the \method-controlled vehicle to lane-change to avoid a crash. In the QEW test, \method~cannot lane-change to avoid a crash and must brake.
\begin{figure}[ht]
\vskip -0.1in
\begin{center}
\centerline{\includegraphics[width=9cm]{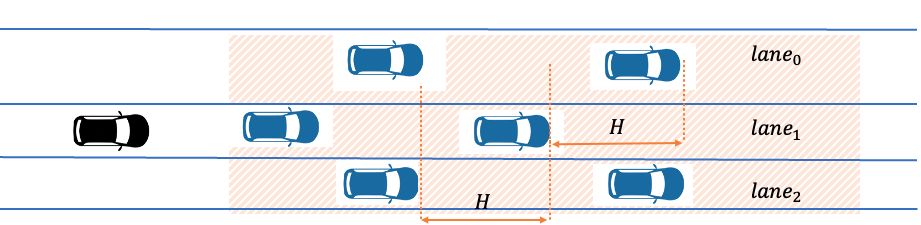}}
\caption{Diagram of the emergency braking scenario in the QEW freeway.} 
\label{brakeqew}
\end{center}
\vspace{-0.5cm}
\end{figure}

\item \textbf{Merging.}
Another scenario is merging from an on-ramp. The relevant portion of the network is illustrated in Figure~\ref{mergeqew}. The human-driver vehicles will keep on generating on the mainline with $H=5$ to make it challenging. The controlled vehicle is generated in the merge lane, and the controlled vehicle needs to merge into the mainline before the lane drop. The merge will be considered failed if the AV still cannot merge to mainline when the simulation time steps have reached the maximum horizon in QEW set in Table~\ref{ddpgtable}. 
\begin{figure}[ht]
\vskip -0.1in
\begin{center}
\centerline{\includegraphics[width=9cm]{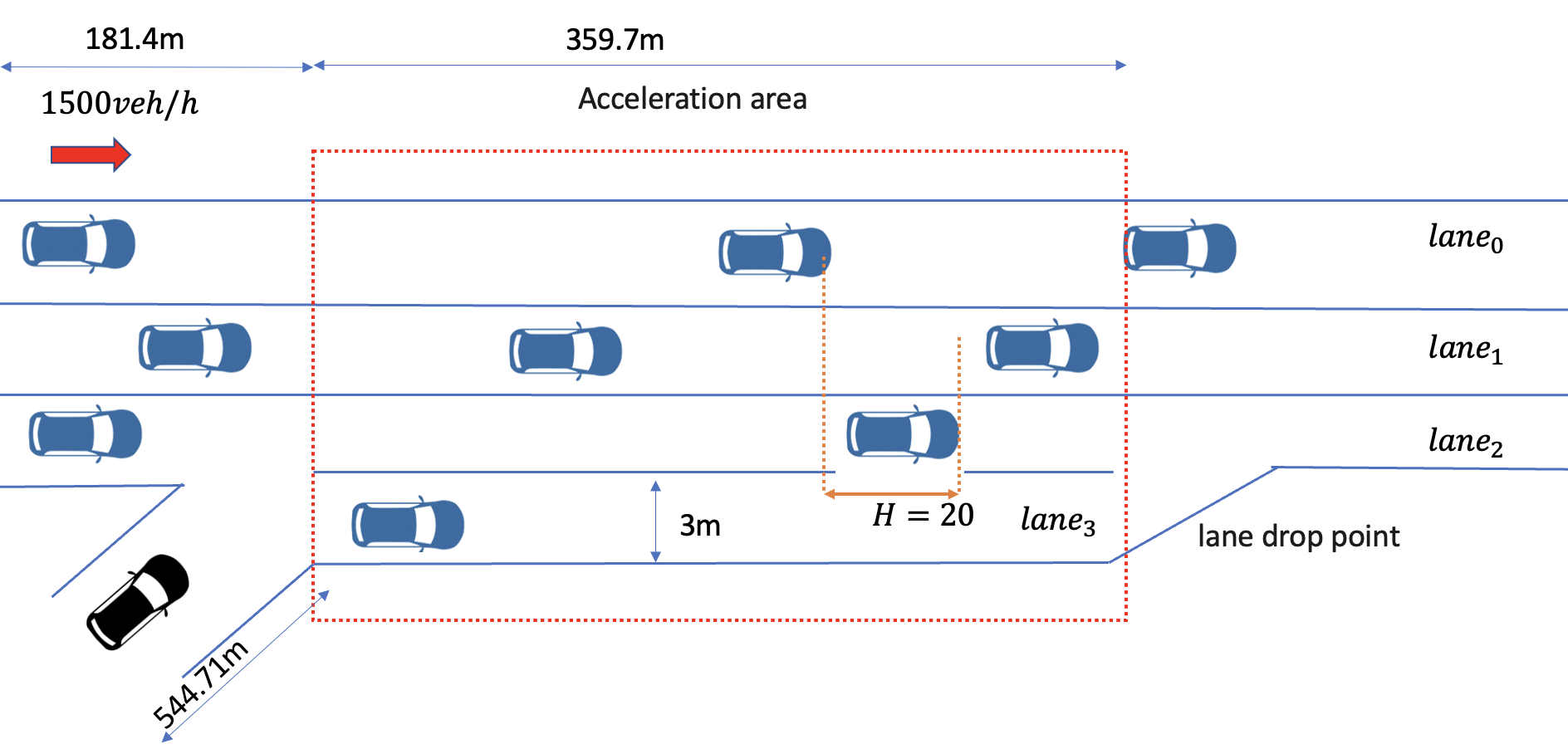}}
\caption{Diagram of the merging scenario.} 
\label{mergeqew}
\end{center}
\vspace{-0.5cm}
\end{figure}
\end{itemize}

\subsection*{Results analysis}
A simulation demo can be found at \cite{demo}.

As shown in Figure~\ref{fig:normalized-stats}, to validate the proposed framework, the training went through an evaluation phase (with 30 random seeds) every 100 epochs. The efficiency, comfort, and routing objectives were tracked during training. The results are compared with two baselines: IDM+MOBIL and Gipps+Greedy To make the results easier to understand, the values are normalized by IDM+MOBIL. At the start of training, the agent tends to blindly accelerate or decelerate. With more training steps, the agent gradually learns to optimize the speed and comfort while mitigating the route miss rate (failure rate).

\begin{figure*}[t]
\centering
{\includegraphics[width=0.9\textwidth]{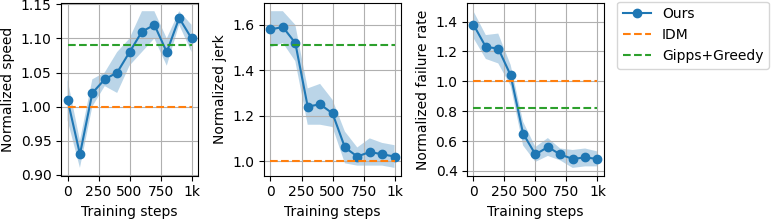}}
\caption{Speed, jerk and route-following miss (failure) rate evaluated at a fixed collection of random seeds, as a function of training iteration. The results are normalized with respect to the IDM baseline.} 
\label{fig:normalized-stats}
\end{figure*}

As in the loop network, the other RL baselines continue to have poor safety performance relative to \method, especially in emergency braking. In all experiments, given accurate estimates of reaction times and braking rates of the surrounding vehicles, \method~is able to avoid crashes altogether.

Table~\ref{table:qew-bypassing} displays the results for the bypassing scenario. This scenario is different from scenarios seen by the agent during training. \method~generalizes well. Overall, as the headway $H$ decreases, making the traffic more crowded, travel efficiency decreases as well. A very narrow $H$ makes lead vehicles hard to bypass, but a very high headway will make lead vehicles easy to bypass. Compared with rule-based controllers, such as IDM+MOBIL or Gipps+Greedy, our controller can achieve higher average speed (better efficiency) with lower jerk (better comfort). Compared with the RL-based controllers, \method~achieves better results in safety, efficiency, comfort and route-following.

Table~\ref{table:qew-emergency} shows that \method~succeeds in having low average jerk and zero crash rate in the emergency-braking scenario.

The merging scenario results in Table~\ref{table:qew-merging} are obtained \emph{without retraining the controller specifically for merging}. Thus, the controller trained on the off-ramp exit scenario is able to successfully generalize to a different route-following scenario (merging from an on-ramp). This highlights an advantage that arises from the generality of our route-following formulation.

Table~\ref{table:differentassumption} shows how efficiency and comfort depend on the reaction times of the controlled and human-driver vehicles (given by $r_E$ and $r_F$, respectively) in the challenging $H=5$ bypassing scenario. Lower reaction times enable higher efficiency control, as the lane-change constraints become less restrictive.

\section{Conclusion}

In conclusion, we address the critical challenge of safety in RL-based autonomous driving controllers. While existing RL-based controllers offer great potential in optimizing efficiency, stability, and comfort, they often lack safety guarantees. To overcome this limitation, we extend our previous work SECRM~\cite{secrm} to  \method~(the Safe, Efficient and Comfortable RL-based driving Model with Lane-Changing), incorporating hard analytic safety constraints on acceleration and lane-change actions. Through extensive simulator experiments, we demonstrate that several representative previously-published RL AV controllers are  prone to crashes, while our \method~controller successfully avoids crashes in both training and testing. Additionally, \method~has excellent performance in efficiency, comfort, safety, and route-following relative to a selection of strong learned and non-learned baselines. Finally, we achieve a good understanding of the long-term equilibrium of a platoon of \method~vehicles moving at constant speed, including obtaining a precise expression for the inter-vehicle gaps in the platoon in terms of its speed and the vehicle configuration.



%\begin{thebibliography}{1}


\bibliographystyle{IEEEtran}
\bibliography{IEEEabrv,Bibliography}

\begin{thebibliography}{10}
\providecommand{\url}[1]{#1}
\csname url@rmstyle\endcsname
\providecommand{\newblock}{\relax}
\providecommand{\bibinfo}[2]{#2}
\providecommand\BIBentrySTDinterwordspacing{\spaceskip=0pt\relax}
\providecommand\BIBentryALTinterwordstretchfactor{4}
\providecommand\BIBentryALTinterwordspacing{\spaceskip=\fontdimen2\font plus
\BIBentryALTinterwordstretchfactor\fontdimen3\font minus \fontdimen4\font\relax}
\providecommand\BIBforeignlanguage[2]{{%
\expandafter\ifx\csname l@#1\endcsname\relax
\typeout{** WARNING: IEEEtran.bst: No hyphenation pattern has been}%
\typeout{** loaded for the language `#1'. Using the pattern for}%
\typeout{** the default language instead.}%
\else
\language=\csname l@#1\endcsname
\fi
#2}}

\bibitem{dqn}
\BIBentryALTinterwordspacing
V.~Mnih, K.~Kavukcuoglu, D.~Silver, A.~A. Rusu, J.~Veness, M.~G. Bellemare, A.~Graves, M.~Riedmiller, A.~K. Fidjeland, G.~Ostrovski, S.~Petersen, C.~Beattie, A.~Sadik, I.~Antonoglou, H.~King, D.~Kumaran, D.~Wierstra, S.~Legg, and D.~Hassabis, ``Human-level control through deep reinforcement learning,'' \emph{Nature}, vol. 518, no. 7540, pp. 529--533, 2015. [Online]. Available: \url{https://doi.org/10.1038/nature14236}
\BIBentrySTDinterwordspacing

\bibitem{lazic2018data}
N.~Lazic, C.~Boutilier, T.~Lu, E.~Wong, B.~Roy, M.~Ryu, and G.~Imwalle, ``Data center cooling using model-predictive control,'' \emph{Advances in Neural Information Processing Systems}, vol.~31, 2018.

\bibitem{li2022decision}
G.~Li, Y.~Yang, S.~Li, X.~Qu, N.~Lyu, and S.~E. Li, ``Decision making of autonomous vehicles in lane change scenarios: Deep reinforcement learning approaches with risk awareness,'' \emph{Transportation Research Part C: Emerging Technologies}, vol. 134, p. 103452, 2022.

\bibitem{secrm}
O.~Elsamadisy, T.~Shi, I.~Smirnov, and B.~Abdulhai, ``Safe, efficient and comfortable reinforcement-learning-based car-following for {AV}s with analytic safety guarantee and dynamic target speed,'' \emph{Transportation Research Record}, vol. 2678(1), pp. 643--661, 2024.

\bibitem{zhang2019discretionary}
S.~Zhang, H.~Peng, S.~Nageshrao, and E.~Tseng, ``Discretionary lane change decision making using reinforcement learning with model-based exploration,'' in \emph{2019 18th IEEE International Conference On Machine Learning And Applications (ICMLA)}.\hskip 1em plus 0.5em minus 0.4em\relax IEEE, 2019, pp. 844--850.

\bibitem{xi2020efficient}
C.~Xi, T.~Shi, Y.~Wu, and L.~Sun, ``Efficient motion planning for automated lane change based on imitation learning and mixed-integer optimization,'' in \emph{2020 {IEEE} 23rd International Conference on Intelligent Transportation Systems (ITSC)}.\hskip 1em plus 0.5em minus 0.4em\relax IEEE, 2020, pp. 1--6.

\bibitem{pan2016modeling}
T.~Pan, W.~H. Lam, A.~Sumalee, and R.~Zhong, ``Modeling the impacts of mandatory and discretionary lane-changing maneuvers,'' \emph{Transportation research part C: emerging technologies}, vol.~68, pp. 403--424, 2016.

\bibitem{shi2019driving}
T.~Shi, P.~Wang, X.~Cheng, C.-Y. Chan, and D.~Huang, ``Driving decision and control for automated lane change behavior based on deep reinforcement learning,'' in \emph{2019 IEEE intelligent transportation systems conference (ITSC)}.\hskip 1em plus 0.5em minus 0.4em\relax IEEE, 2019, pp. 2895--2900.

\bibitem{hoel2019combining}
C.-J. Hoel, K.~Driggs-Campbell, K.~Wolff, L.~Laine, and M.~J. Kochenderfer, ``Combining planning and deep reinforcement learning in tactical decision making for autonomous driving,'' \emph{IEEE transactions on intelligent vehicles}, vol.~5, no.~2, pp. 294--305, 2019.

\bibitem{cao2020highway}
Z.~Cao, D.~Yang, S.~Xu, H.~Peng, B.~Li, S.~Feng, and D.~Zhao, ``Highway exiting planner for automated vehicles using reinforcement learning,'' \emph{IEEE Transactions on Intelligent Transportation Systems}, vol.~22, no.~2, pp. 990--1000, 2020.

\bibitem{ye2020automated}
F.~Ye, X.~Cheng, P.~Wang, C.-Y. Chan, and J.~Zhang, ``Automated lane change strategy using proximal policy optimization-based deep reinforcement learning,'' in \emph{2020 IEEE Intelligent Vehicles Symposium (IV)}.\hskip 1em plus 0.5em minus 0.4em\relax IEEE, 2020, pp. 1746--1752.

\bibitem{udatha2023reinforcement}
S.~Udatha, Y.~Lyu, and J.~Dolan, ``Reinforcement learning with probabilistically safe control barrier functions for ramp merging,'' in \emph{2023 IEEE International Conference on Robotics and Automation (ICRA)}.\hskip 1em plus 0.5em minus 0.4em\relax IEEE, 2023, pp. 5625--5630.

\bibitem{vienna}
B.~Vanholme, D.~Gruyer, B.~Lusetti, S.~Glaser, and S.~Mammar, ``Highly automated driving on highways based on legal safety,'' \emph{IEEE Transactions on Intelligent Transportation Systems}, vol.~14, no.~1, pp. 333--347, 2013.

\bibitem{zhu2020safe}
M.~Zhu, Y.~Wang, Z.~Pu, J.~Hu, X.~Wang, and R.~Ke, ``Safe, efficient, and comfortable velocity control based on reinforcement learning for autonomous driving,'' \emph{Transportation Research Part C: Emerging Technologies}, vol. 117, p. 102662, 2020.

\bibitem{yen2020proactive}
Y.-T. Yen, J.-J. Chou, C.-S. Shih, C.-W. Chen, and P.-K. Tsung, ``Proactive car-following using deep-reinforcement learning,'' in \emph{2020 IEEE 23rd International Conference on Intelligent Transportation Systems (ITSC)}.\hskip 1em plus 0.5em minus 0.4em\relax IEEE, 2020, pp. 1--6.

\bibitem{shi2022bilateral}
T.~Shi, Y.~Ai, O.~ElSamadisy, and B.~Abdulhai, ``Bilateral deep reinforcement learning approach for better-than-human car-following,'' in \emph{2022 IEEE 25th International Conference on Intelligent Transportation Systems (ITSC)}.\hskip 1em plus 0.5em minus 0.4em\relax IEEE, 2022, pp. 3986--3992.

\bibitem{gipps}
P.~Gipps, ``A behavioural car-following model for computer simulation,'' \emph{Transportation Research Part B: Methodological}, vol.~15, no.~2, pp. 105--111, 1981.

\bibitem{safe-rl-survey}
W.~Zhao, T.~He, R.~Chen, T.~Wei, and C.~Liu, ``State-wise safe reinforcement learning: A survey,'' \emph{arXiv preprint arXiv:2302.03122}, 2023.

\bibitem{ddpg}
\BIBentryALTinterwordspacing
T.~P. Lillicrap, J.~J. Hunt, A.~Pritzel, N.~Heess, T.~Erez, Y.~Tassa, D.~Silver, and D.~Wierstra, ``Continuous control with deep reinforcement learning,'' in \emph{4th International Conference on Learning Representations, {ICLR} 2016, San Juan, Puerto Rico, May 2-4, 2016, Conference Track Proceedings}, Y.~Bengio and Y.~LeCun, Eds., 2016. [Online]. Available: \url{http://arxiv.org/abs/1509.02971}
\BIBentrySTDinterwordspacing

\bibitem{sumo}
P.~{\'A}. L{\'o}pez, M.~Behrisch, L.~Bieker-Walz, J.~Erdmann, Y.-P. Fl{\"o}tter{\"o}d, R.~Hilbrich, L.~L{\"u}cken, J.~Rummel, P.~Wagner, and E.~WieBner, ``Microscopic traffic simulation using {SUMO},'' \emph{2018 21st International Conference on Intelligent Transportation Systems (ITSC)}, pp. 2575--2582, 2018.

\bibitem{discdynsys}
O.~Galor, \emph{{D}iscrete {D}ynamical {S}ystems}.\hskip 1em plus 0.5em minus 0.4em\relax Springer-Verlag, 2007.

\bibitem{dynsys}
J.~Guckenheimer and P.~Holmes, \emph{{N}onlinear {O}scillations, {D}ynamical {S}ystems, and {B}ifurcations of {V}ector {F}ields}, ser. {A}pplied {M}athematical {S}ciences.\hskip 1em plus 0.5em minus 0.4em\relax Springer-Verlag, 1983, vol.~42.

\bibitem{idm}
Treiber, Hennecke, and Helbing, ``Congested traffic states in empirical observations and microscopic simulations,'' \emph{Physical Review E: Statistical Physics, Plasmas, Fluids, and Related Interdisciplinary Topics}, vol. 62 2 Pt A, pp. 1805--24, 2000.

\bibitem{treiber2016mobil}
A.~Kesting, M.~Treiber, and D.~Helbing, ``General lane-changing model {MOBIL} for car-following models,'' \emph{Transportation Research Record}, vol. 1999(1), pp. 86--94, 2007.

\bibitem{erdmann2015sumo}
J.~Erdmann, ``{SUMO}’s lane-changing model,'' in \emph{Modeling Mobility with Open Data: 2nd SUMO Conference 2014 Berlin, Germany, May 15-16, 2014}.\hskip 1em plus 0.5em minus 0.4em\relax Springer, 2015, pp. 105--123.

\bibitem{demo}
T.~Shi, ``{SECRM-2D} demo,'' 2023, \small \texttt{https://youtu.be/J5UyJXGAGX0}, Last accessed May 20, 2023.

\end{thebibliography}

\begin{appendices}
\section{Asymptotic Stability of the Follower With Constant-Speed Leader System}
\label{app:jac-eigen-detail}
Here, we would like to verify that the eigenvalues of the Jacobian matrix at the equilibrium point $(g^*, v^*)$
\[
    \lambda_{\pm} = \frac{w \pm \sqrt{w^2 - 2wd\dt - d^2\dt^2}}{2w + dr}
\]
satisfy $\left| \lambda_\pm \right| < 1$ whenever $w > 0$. We consider two cases based on whether the eigenvalues are real or complex.

\paragraph{Case $w^2 - 2wd\dt - d^2\dt^2 \geq 0$} In this case, the eigenvalues are real. First, we show that $\lambda_+ < 1$. This follows from
\begin{align*}
    \frac{w + \sqrt{w^2 - 2wd\dt - d^2\dt^2}}{2w + d\dt} &< \frac{w + \sqrt{w^2 - 2wd\dt + d^2 \dt^2}}{2w + d\dt} \\
    &= \frac{w + (w - d\dt)}{2w + d\dt} \\
    &= \frac{2w - d\dt}{2w + d\dt} < 1, \quad \mbox{as } d\dt > 0
\end{align*}
Now, we show that $\lambda_- > -1$. Since $2w + d\dt > 0$, we have $2(2w+d\dt)^2 > 0$ also. Therefore,
\[2(2w+d\dt)^2 + (w^2 - 2wd\dt - d^2\dt) > w^2 - 2wd\dt - d^2\dt\]
Collecting the terms on the left side, we have $(3w+d\dt)^2 > w^2 - 2wd\dt - d^2\dt$. Taking square roots of both sides and rearranging,
\[ w - \sqrt{w^2 - 2wd\dt - d^2\dt} > -2w - d\dt \]
Dividing both sides by $2w + d\dt$ shows that $\lambda_i > -1$. Then, the chain of inequalities
\[ -1 < \lambda_i < \lambda_+ < 1 \]
shows that $\left| \lambda_\pm \right| < 1$, as desired.

\paragraph{Case $w^2 - 2wd\dt - d^2\dt^2 < 0$} In this case, the eigenvalues are complex conjugates. It is sufficient to check that either of the eigenvalues have norm $< 1$, as complex conjugation is norm-preserving.

The squared norm of $\lambda_+$ is
\[
    \frac{w^2 + (d^2\dt^2 + 2wd\dt - w^2)}{(2w + d\dt)^2} = \frac{d^2 \dt^2 + 2wd\dt}{(2w + dr)^2}
\]
To show that $\left| \lambda_+ \right|^2 < 1$ it is equivalent to check that $d^2 \dt^2 + 2wd\dt < (2w + d\dt)^2$. Expanding the square and rearranging, the latter is equivalent to $0 < 4w^2 + 2wdr$, which is true whenever $w > 0$. It follows that $\left| \lambda_\pm \right| < 1$, as desired.

\end{appendices}

\begin{IEEEbiography}
 [{\includegraphics[width=1in,height=1.25in,clip,keepaspectratio]{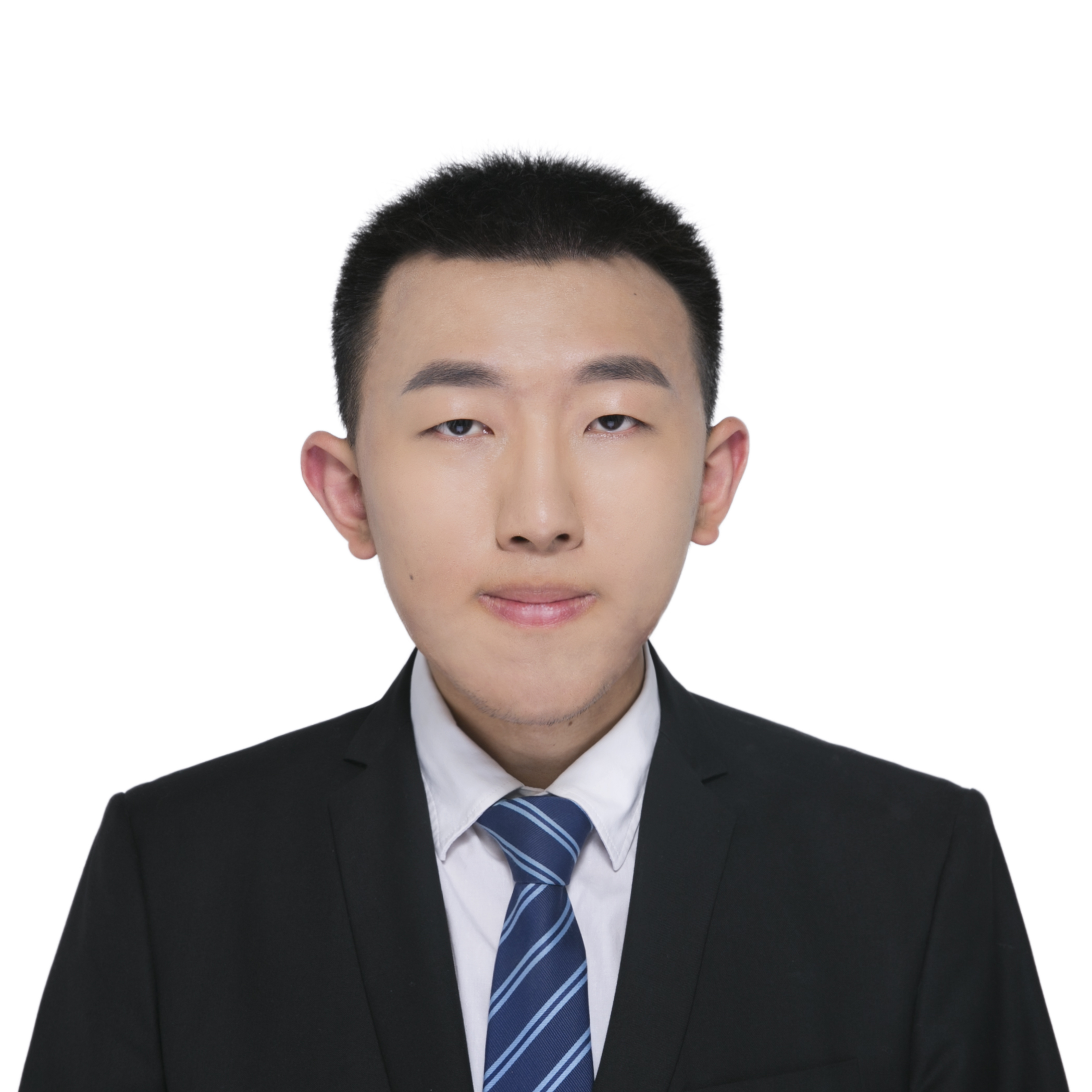}}]{Tianyu Shi} received the B.Sc. degree in Vehicle Engineering from Beijing Institute of Technology, Beijing, China, in 2019. He received master's degree from  McGill University, Montreal, Canada, in 2021. He is currently pursuing Ph.D. degree in Transportation Engineering, 
University of Toronto, Toronto, Canada.
His current research centers on reinforcement learning and intelligent transportation systems.
\end{IEEEbiography}

\begin{IEEEbiography}[{\includegraphics[width=1in,height=1.25in,clip,keepaspectratio]{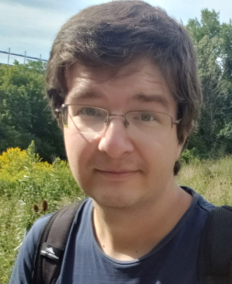}}]{Ilia Smirnov}
received the B.Sc. degree in mathematics and physics from University of Toronto, Canada, in 2010, and the M.Sc. and Ph.D. degrees in pure mathematics from Queen's University, Canada, in 2012 and 2020. He is a research associate in Civil Engineering at the University of Toronto. His research interests include artificial intelligence and optimization in intelligent transportation systems, and algebraic geometry.
\end{IEEEbiography}

\begin{IEEEbiography}[{\includegraphics[width=1in,height=1.25in,clip,keepaspectratio]{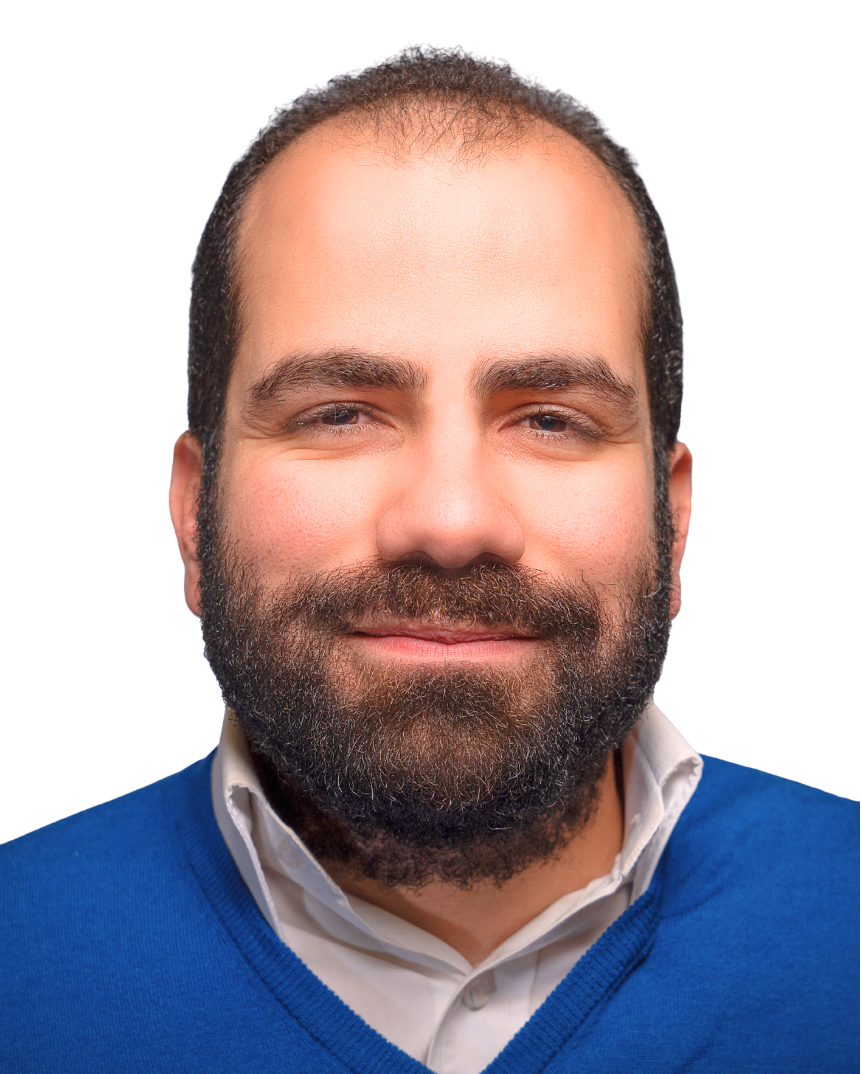}}]{Omar ElSamadisy}
 received both his B.Sc. and M.Sc. from the Arab Academy for Science and Technology, Egypt, in 2014 and 2017, respectively. He is currently a Ph.D. candidate at the University of Toronto, Canada, focusing on reinforcement learning applications in intelligent transportation systems. His research aims to optimize urban freeway operations through the integration of artificial intelligence and transportation engineering principles.
\end{IEEEbiography}

\begin{IEEEbiography}[{\includegraphics[width=1in,height=1.25in,clip,keepaspectratio]{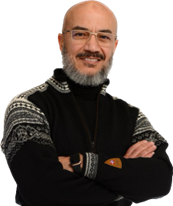}}]{Baher Abdulhai}
received the Ph.D. degree in engineering from the University of California, Irvine, CA, USA, in 1996. He is a Professor in Civil Engineering at the University of Toronto, ON, Canada. He has 35 years of experience in transportation systems engineering and Intelligent Transportation Systems (ITS). He is the founder and Director of the Toronto Intelligent Transportation System Center, and the founder and co-Director of the i-City Center for Automated and Transformative Transportation Systems (iCity-CATTS).
He received several awards including IEEE Outstanding Service Award, Teaching Excellence award, and research awards from Canada Foundation for Innovation, Ontario Research Fund, and Ontario Innovation Trust.  He served on the Board of Directors of the Government of Ontario (GO) Transit Authority from 2004 to 2006.  He served as a Canada Research Chair (CRC) in ITS from 2005 to 2010. His research team won international awards including the International Transportation Forum innovation award in 2010 (Hossam Abdelgawad), IEEE ITS 2013 (Samah El-Tantawy) and INFORMS 2013 (Samah El-Tantawy). In 2015 he has been inducted as a Fellow of the Engineering Institute of Canada (EIC). In 2018, he won the prestigious CSCE Sandford Fleming (Career Achievement) Award for his contribution to transportation in Canada. He has been elected Fellow of the Canadian Academy of Engineering in 2020. In 2021, he won the Ontario Professional Engineers Awards (OPEA) Engineering Medal for career Engineering Excellence.
\end{IEEEbiography}
\vfill

\end{document}